\documentclass[11pt]{article}

\RequirePackage{bm}

\usepackage{geometry}
\usepackage{amsmath}
\usepackage{amsfonts}
\usepackage{amsthm}
\usepackage{caption}
\usepackage{subcaption}
\usepackage{pgfplots}
\usepackage{float}
\usepackage{hyperref}
\usepackage{mathtools}
\usepackage{makecell}
\usepackage{bbm}
\usepackage[ruled,vlined]{algorithm2e}
\usepackage{multirow}
\usepackage{MnSymbol}
\usepackage{setspace}

\newtheorem{theorem}{Theorem}[section]

\newtheorem{proposition}{Proposition}

\def\bbR{\mathbb{R}}
\def\dsp{\displaystyle}

\title{From Data to Uncertainty Sets: a Machine Learning Approach}

\author{
  Dimitris Bertsimas \\
  \small dbertsim@mit.edu
  \and
  Benjamin Boucher \\
  \small bboucher@mit.edu
}
\date{}

\begin{document}

\maketitle

\begin{abstract}
% Enter your abstract
Existing approaches of prescriptive analytics---where inputs of an optimization model can be predicted by leveraging covariates in a machine learning model---often attempt to optimize the mean value of an uncertain objective. However, when applied to uncertain constraints, these methods rarely work because satisfying a crucial constraint in expectation may result in a high probability of violation. To remedy this, we leverage robust optimization to protect a constraint against the uncertainty of a machine learning model's output. To do so, we design an uncertainty set based on the model's loss function.
Intuitively, this approach attempts to minimize the uncertainty around a prediction. Extending guarantees from the robust optimization literature, we derive strong guarantees on the probability of violation. On synthetic computational experiments, our method requires uncertainty sets with radii up to one order of magnitude smaller than those of other approaches.
\end{abstract}

\onehalfspacing

\section{Introduction}

Optimization models often incorporate parameters that are estimated by a machine learning model, such as a cost, demand, or a duration. Formally, consider a constraint in an optimization problem:
\begin{equation} \label{deterministic_constraint}
    \dsp g \left( \bm{y}, \bm{x} \right) \leq 0,
\end{equation}

\noindent where $\dsp \bm{x}$ are the decision variables, $\dsp \bm{y}$ are the parameters of the model, and $\dsp g$ is a function mapping to the real numbers.

When covariates providing information about the model's parameters $\dsp \bm{y}$ are available, the modeler might wish to make use of them by training a machine learning model to estimate the value of $\dsp \bm{y}$ given a new observation of the covariates. However, these predictions are inherently uncertain---a key question is, therefore, how the modeler should go about treating this uncertainty.

One approach is to ignore the uncertainty entirely and use the predictions from the machine learning model as ground truth in the optimization stage. This two-staged approach, which Elmachtoub and Grigas (2022) \cite{elmachtoub2022smart} dub ``predict-then-optimize,'' fails to take into account the uncertainty surrounding the predictions of the model.

Many approaches have tackled the case of an uncertain objective function with a focus on minimizing the expected objective value. Elmachtoub and Grigas (2022) \cite{elmachtoub2022smart}, for example, propose to train models with a loss function minimizing the suboptimality of the downstream optimization model. They focus their study on the prediction of a linear objective cost vector. Furthermore, a core component of their approach is to have a single model predicting $\dsp \bm{y}$, which makes training larger models complicated. In contrast, Bertsimas and Kallus (2020) \cite{bertsimas2020predictive} propose to use machine learning models to describe new samples as a weighted sum of training observations. This method does not suffer from the scaling issues from the previous technique but coincides with the predict-then-optimize approach in the case of a linear objective.
Notice that neither approach extends to the more general setting of uncertainty embedded in a constraint. Indeed, a constraint satisfied in expectation could be frequently violated.

The most common way to protect an optimization problem against the uncertainty of a parameter is to leverage robust optimization, which proposes to solve constraint (\ref{deterministic_constraint}) for a set $\dsp \mathcal{U}$ of values for $\dsp \bm{y}$:
\begin{equation}
    \label{constraint}
    \dsp g \left( \bm{y}, \bm{x} \right) \leq 0, \, \forall \bm{y} \in \mathcal{U}.
\end{equation}

A variety of methods have been employed to build $\dsp \mathcal{U}$ effectively in the context of available covariates. Sun et al. (2023) \cite{sun2023predict} suppose having access to a good prediction model, whose residuals are then used to train a model to calibrate and construct box and ellipsoidal uncertainty sets. Tulabandhula and Rudin (2014) \cite{tulabandhula2014robust} propose a two-stage approach: first build an uncertainty set to protect against the difference between the current model and the theoretically best model, then build a set to protect against the predictions of the best model. Both of these approaches obtain probabilistic guarantees by leveraging the notion of Rademacher complexity from the statistical learning literature and therefore require a vast pool of validation data to draw any significant guarantees. As an illustrative example, to obtain guarantees that the constraint holds 90\% of the time for a new observation, Sun et al. (2023) require having access to at least 240,000 validation data points, while necessitating a large radius for their box uncertainty set. A modeler might prefer guarantees that are more practical when additional hypotheses hold, such as independence of the uncertain parameters given the covariates.

\subsection{Contributions}

In this paper, we propose a new uncertainty set informed by supervised machine learning models. Our contributions are as follows:

\begin{enumerate}
  \item We motivate and introduce a new uncertainty set that incorporates the loss function of the trained machine learning model as a measure of the uncertainty around the predictions.
  \item We connect our uncertainty set with uncertainty quantification from the deep learning literature (Nix and Weigend (1994) \cite{nix1994estimating}, Bishop (1994) \cite{bishop1994mixture} , and Skafte et al. (2019) \cite{skafte2019reliable}). Remarkably, we show that our uncertainty sets can be viewed as a generalization of the ellipsoidal uncertainty set proposed by Ben-Tal (2000) \cite{ben2000robust} to the context of covariate-driven robust optimization.
  \item We demonstrate that our uncertainty set can be solved tractably when the constraint function $\dsp g$ is concave in the uncertainty.
  \item We provide guarantees for the probability of violation in general. Additionally, we provide improved guarantees for the most common loss function when the components of the uncertainty conditioned on the covariates are independent.
  \item We run extensive experiments on synthetic data, and show that our method requires uncertainty sets with radii up to one order of magnitude smaller than those of other approaches to obtain the same probabilistic guarantees.
\end{enumerate}

The structure of the paper is as follows. In Section \ref{sec:using_loss_functions}, we motivate our uncertainty set by incorporating the loss function of a machine learning model, and explore how to solve a constraint leveraging the set. In Section \ref{sec:probabilistic_guarantees}, we develop probabilistic guarantees that ensure our uncertainty sets will avoid out-of-sample violations of the constraint, up to a desired threshold. We then provide comprehensive computational experiments in Section \ref{sec:computational_experiments}.

\subsection{Notations and Definitions}

Boldfaced letters $\displaystyle \bm{x}$ denote vectors and matrices while ordinary lowercase letters $\displaystyle x$ denote scalars. We add a tilde to denote a random vector $\displaystyle \Tilde{\bm{x}}$ and a random variable $\displaystyle \Tilde{x}$. The $\displaystyle i$\textsuperscript{th} coordinate of the vector $\displaystyle \bm{x}$ is $\displaystyle x_{i}$.

Let $\displaystyle n \in \mathbb{N}^{*}$ be the dimension of our decision variables $\displaystyle \bm{x}$ and $\displaystyle m \in \mathbb{N}^{*}$ be the dimension of our uncertain parameters. Denote by $\dsp \bm{y} \in \mathcal{Y} \subset \bbR^{m}$ the uncertain parameters which belong to a set $\dsp \mathcal{Y}$, and $\dsp g: \bbR^{m} \times \bbR^{n} \to \bbR$ be the function used in the constraint of an optimization formulation like (\ref{constraint}). We write $\dsp [m] = \left\{ 1, \, 2, \, \cdots, \, m \right\}$. We note for two vectors $\left(\bm{a}, \bm{b} \right) \in \bbR^{m} \times \bbR^{m}, \bm{a} \cdot \bm{b} = \left( a_{i} b_{i} \right)_{i \in [m]}$ the vector composed of their entry-wise multiplications and $\dsp \frac{\bm{a}}{\bm{b}} = \left( \frac{a_{i}}{b_{i}} \right)_{i \in [m]}$ the vector composed of their entry-wise divisions, when $\dsp \bm{b} > 0$.

Let $\dsp \mathcal{X} \subset \bbR^{p}$ be the space of covariates which provide information about the parameter $\dsp \bm{y} \in \mathcal{Y}$. We denote by $\dsp f: \mathcal{X} \to \mathcal{Y} \in \mathcal{F}$ a machine learning model by which the modeler predicts $\dsp \bm{y} \in \mathcal{Y}$ given a new observation of the covariates $\dsp \bm{X} \in \mathcal{X}$. The prediction model is chosen within the function class $\dsp \mathcal{F}$ so as to minimize a loss metric $\ell:\mathcal{Y} \times \mathcal{Y} \to \bbR$.

Furthermore, we define a probability measure $\dsp \mathbb{P}$ on the space $\dsp \mathcal{X} \times \mathcal{Y}$, and denote its marginal distribution $\dsp \mathbb{P} \left( \cdot \, \middle| \, \Tilde{\bm{X}} \right)$. Similarly, we denote by $\dsp \mathbb{E}$ the expectation with regards to the measure $\dsp \mathbb{P}$.

\section{Motivation with Loss Functions} \label{sec:using_loss_functions}

In this section, we introduce our proposed uncertainty set. We first discuss the intuition behind our proposal and how to practically construct this uncertainty set for both regression and classification settings.

A loss function $\dsp \ell$ quantifies the quality of a prediction with respect to its true value. Naturally, it also reflects the uncertainty associated with a prediction: a high loss indicates greater uncertainty, while a low loss suggests higher confidence in the model’s output.
The final model $\dsp f^{\star}$ is selected as the function within the given function class $\dsp \mathcal{F}$ that minimizes this loss. The underlying assumption is that $\dsp f^{\star}$ will continue to approximate the minimizer of $\dsp \mathcal{F}$, even on new data. Consequently, the predictor is chosen to reduce the uncertainty of our predictions. This therefore motivates the following uncertainty set, where $\dsp \bm{X} \in \bbR^{p}$ is a new instance of covariates and $\dsp \rho \in \mathbb{R}$:
\begin{equation}
    \label{uncertainty_set}
    \dsp \mathcal{U} \left( \bm{X} \right) = \left\{ \bm{y} \in \bbR^{m} \, \middle| \, \ell \left( \bm{y}, f^{\star} \left( \bm{X} \right) \right) \leq \rho \right\},
\end{equation}

\noindent i.e., we are protecting against all realizations of the uncertain parameter within some tolerated loss from the model's prediction.

An advantage of the uncertainty set (\ref{uncertainty_set}) is that it can be deployed ``off the shelf'' whenever a modeler wants to protect against the uncertainty of a machine learning model utilized in their optimization problem. Notice that (\ref{uncertainty_set}) requires neither training a tailored model nor having access to a validation set. However, obtaining probabilistic guarantees will require a validation set, as we will discuss in Section \ref{sec:probabilistic_guarantees}. Furthermore, in regression applications where the variance is heteroscedastic, the modeler might favor taking into account this varying level of uncertainty by following the guidance developed in Section \ref{heteroscedastic_regression}.

Note that the uncertainty set (\ref{uncertainty_set}) depends on the realization of $\dsp \Tilde{\bm{X}}$, meaning the decision variables $\dsp \bm{x}$ are random variables themselves, and a function of the predicted values. For simplicity of notation, we will omit that dependence.

\subsection{Combining Loss Functions}

Although the uncertainty set (\ref{uncertainty_set}) is straightforward to build when the whole vector $\dsp \bm{y}$ is predicted at once and is evaluated through a single loss function $\dsp \ell$, there are instances for which multiple outputs, each with their own loss, are aggregated into the vector $\dsp \bm{y}$.

This could happen if the modeler favors training separate models $\dsp f^{\star}_{j}$, belonging to a function class $\dsp \mathcal{F}_{j}$, to predict each component (or, more generally, subsets of components $\dsp \bm{y}^{(j)}$) by using a loss $\dsp \ell_{j}$. For instance, the modeler could favor splitting the predictions into multiple models to diminish the complexity and training time of each individual model. Another instance for which the modeler might not have a combined loss function readily available is when the vector's components $\dsp y^{(j)}$ represent similar quantities that can be predicted by the same model (with potentially different inputs $\bm{X}^{j}$). For example, when each component $\dsp y^{(j)}$ represents the return of an asset, a single model can be trained to predict a return given covariates of the asset using a loss function $\dsp L$. That model can then be used to predict all the components of $\dsp \bm{y}$. In this case, $\dsp \ell_{j} = L$ for all $\dsp j \in [m]$.

To utilize the proposed uncertainty set (\ref{uncertainty_set}) the modeler would therefore have to aggregate all loss functions. Although the modeler might be able to leverage some domain-specific insight to combine these loss functions, this assumption is not general. This therefore raises the question of how the modeler should proceed about combining these loss functions more generally. Fortunately, this problem has been extensively explored in the multi-task learning literature. See for example an overview by Ruder (2017) \cite{ruder2017overview}.

We propose to follow a popular approach presented by Kendall et al. (2018) \cite{kendall2018multi} that explores weighting each loss function $\dsp \ell_{j}$ by adequate weights $\dsp \alpha^{j} > 0$. Leveraging the homoscedastic uncertainty of models to weight each loss function enables the authors to train a combined model despite the vastly different scales of the original losses. Formally, to combine multiple losses, they weight a loss $\dsp \ell_{j}$ with a parameter $\dsp \sigma_{j}$ to obtain the following losses when $\dsp f_{j}$ is a regression:
\begin{equation*}
    \dsp \ell_{j}' \left( \bm{y}^{(j)}, f_{j} \left( \bm{X} \right), \sigma_{j} \right)
    = \frac{\ell_{j} \left(\bm{y}^{(j)}, f_{j} \left( \bm{X} \right) \right)}{\sigma_{j}^{2}} + \log \left( \sigma_{j}^{2} \right),
\end{equation*}

\noindent and the following loss when $\dsp f_{j}$ is a classification (usually translated to probabilities through a softmax function):
\begin{equation*}
    \dsp \ell_{j}' \left( \bm{y}^{(j)}, f_{j} \left( \bm{X} \right), \sigma_{j} \right)
    = 2\frac{\ell_{j} \left(\bm{y}^{(j)}, f_{j} \left( \bm{X} \right) \right)}{\sigma_{j}^{2}} + \log \left( \sigma_{j}^{2} \right),
\end{equation*}

\noindent which they then use to aggregate into a single loss function:
\begin{equation}
    \label{combined_loss}
    \dsp \ell' \left( \bm{y}, f^{\star} \left( \bm{X} \right), \bm{\sigma} \right) = \sum_{j}\ell_{j}' \left( \bm{y}^{(j)}, f_{j}^{*} \left( \bm{X} \right), \sigma_{j} \right). 
\end{equation}

Notice that each $\dsp \sigma_{j}$ is not a function of the features but a learned weight between losses of vastly different orders of magnitude.

From here the modeler could choose to retrain new models using the aggregated loss (\ref{combined_loss}). However, if we assume the modeler already trained each model $\dsp f^{\star}_{j}$ separately minimizing their respective loss $\dsp \ell_{j}$, the minimizer of (\ref{combined_loss}) can be calculated using the following theorem.

\begin{theorem}
    If $\dsp \forall j \in [m], \, \mathbb{E} \left[ \ell_{j} \left( \Tilde{y}^{(j)}, f_{j}^{*} \left( \Tilde{ \bm{X} } \right) \right) \right] > 0$, the minimizer of $\dsp \left( f, \bm{\sigma} \right) \mapsto \mathbb{E} \left[ \ell' \left( \Tilde{\bm{y}}, f \left( \Tilde{\bm{X}} \right), \bm{\sigma} \right) \right]$ on $\dsp \prod_{j} \mathcal{F}_{j} \times \left(\mathbb{R}^{*}_{+}\right)^{m}$ is $\dsp \left( f^{\star},\bm{\sigma^{*}} \right)$ where for all $\dsp j$:
    \begin{equation*}
        \dsp \left( \sigma^{*}_{j} \right)^{2} =
        \begin{cases}
        \mathbb{E} \left[ \ell_{j} \left( \Tilde{\bm{y}}^{(j)}, f_{j}^{*} \left( \Tilde{\bm{X}} \right) \right) \right], \text{ if } f^{\star}_{j} \text{ is a regression}, \\
        2 \mathbb{E} \left[ \ell_{j} \left( \Tilde{\bm{y}}^{(j)}, f_{j}^{*} \left( \Tilde{\bm{X}} \right) \right) \right], \text{ if } f^{\star}_{j} \text{ is a classification.}
    \end{cases}
    \end{equation*}
\end{theorem}

\begin{proof}{Proof}
    For all $\dsp j$, let $\dsp \epsilon_{j} = \begin{cases}
    1, \text{ if } f_{j} \text{ is a regression}, \\
    2, \text{ if } f_{j} \text{ is a classification}.
    \end{cases}$
    By linearity of the expectation, we find that:
    \begin{equation*}
        \dsp \mathbb{E} \left[ \ell' \left( \bm{y}, f \left( \bm{X} \right), \bm{\sigma} \right) \right]
        = \sum_{j} \epsilon_{j} \frac{ \mathbb{E} \left[ \ell_{j} \left(\bm{y}^{(j)}, f_{j} \left( \bm{X} \right) \right) \right] }{\sigma_{j}^{2}} + \log \left( \sigma_{j}^{2} \right).
    \end{equation*}
    To minimize the sum, we can minimize each term at a time. Furthermore, for all $\dsp j$:
    \begin{align*}
        \dsp &\min_{f_{j}, \sigma_{j} \in \mathcal{F}_{j} \times \mathbb{R}^{*}_{+}} \epsilon_{j} \frac{ \mathbb{E} \left[ \ell_{j} \left( \Tilde{\bm{y}}^{(j)}, f_{j} \left( \Tilde{\bm{X}} \right) \right) \right] }{\sigma_{j}^{2}} + \log \left( \sigma_{j}^{2} \right) \\
        = &\min_{\sigma_{j} \in \mathbb{R}^{*}_{+}} \epsilon_{j} \frac{ \min_{f \in \mathcal{F}_{j}} \mathbb{E} \left[ \ell_{j} \left( \Tilde{\bm{y}}^{(j)}, f_{j} \left( \Tilde{\bm{X}} \right) \right) \right] }{\sigma_{j}^{2}} + \log \left( \sigma_{j}^{2} \right).
    \end{align*}

    We have that for all $\dsp \alpha > 0, 
    \sigma_{j} \mapsto \frac{\alpha}{\sigma_{j}^{2}} + \log \left( \sigma_{j}^{2} \right)$ is convex. By differentiating, we can prove that its minimum is attained when $\dsp \sigma_{j}^{2} = \alpha$, which concludes the proof.
\end{proof}

The above theorem implies that the aggregated loss function we should consider is:
\begin{equation*}
    \dsp \sum_{j} \frac{\ell_{j} \left(\bm{y}^{(j)}, f_{j}^{*} \left( \bm{X} \right) \right)}{\mathbb{E} \left[ \ell_{j} \left(\Tilde{\bm{y}}^{(j)}, f_{j}^{*} \left( \Tilde{\bm{X}} \right) \right) \right]} + \sum_{j} \epsilon_{j} \log \left( \sigma_{j}^{2} \right).
\end{equation*}

Since the second sum is a constant, we can therefore study the scaled uncertainty set:
\begin{equation}
    \label{uncertainty_set_scaled}
    \dsp \mathcal{U} \left( \bm{X} \right) = \left\{ \bm{y} \in \bbR^{m} \, \middle| \, \sum_{j} \frac{\ell_{j} \left(\bm{y}^{(j)}, f_{j}^{*} \left( \bm{X} \right) \right)}{\mathbb{E}  \left[ \ell_{j} \left(\Tilde{\bm{y}}^{(j)}, f_{j}^{*} \left( \Tilde{\bm{X}} \right) \right) \right]} \leq \rho' \right\},
\end{equation}

\noindent where $\dsp \rho' = \rho -\sum_{j} \epsilon_{j} \log \left( \Hat{\sigma}_{j}^{2} \right)$ is a constant.

Notice that the expectations in the set (\ref{uncertainty_set_scaled}), can all be estimated on the training set. Furthermore, if the same prediction model is used to predict all the components, i.e., for all $\dsp j$, $\dsp f_{j}^{*}$ are equal and $\dsp \ell_{j} = L$, then the uncertainty set (\ref{uncertainty_set}) is simply scaled by $\dsp \mathbb{E} \left[ L \left(\Tilde{\bm{y}}^{(1)}, f_{1}^{*} \left( \Tilde{\bm{X}} \right) \right) \right]$. Reciprocally, models fully predicting the vector $\dsp \bm{y}$ can usually be decomposed as a sum of individual losses $\dsp \ell_{j}$, i.e., $\dsp \ell \left( \bm{y}, f^{\star} \left( \bm{X} \right) \right) = \sum_{j} \ell_{j} \left( \bm{y}^{(j)}, f^{\star}_{j} \left( \bm{X} \right) \right)$.

We now take a closer look at how to solve (\ref{constraint}) when utilizing our uncertainty set (\ref{uncertainty_set}) for the most common loss functions.

\subsection{Solving the Robust Optimization Problem over the Uncertainty Set}

An advantage of loss functions is that they are often convex functions in their first arguments and therefore tractable to handle. We explore below some strategies to solve a constraint leveraging the uncertainty set (\ref{uncertainty_set}) for both classification and regression. We also highlight some machine learning algorithms which utilize these losses.

\medbreak
\noindent \textbf{Classification}
\medbreak

\noindent The cross-entropy loss, the hinge loss function, and the misclassification loss function are among the most utilized in the machine learning literature for classification tasks. These losses are the basis of a wide range of machine learning models from Logistic Regression, Deep Neural Networks, Support Vector Machine (Cortes and Vapnik (1995) \cite{cortes1995support}) to Optimal Classification Trees (Bertsimas and Dunn (2017) \cite{bertsimas2017optimal}).

Solving constraint (\ref{constraint}) with the uncertainty set (\ref{uncertainty_set}) in the classification case implies having to deal with integer variables, as classifications are inherently discrete. Fortunately, when $\dsp g$ is uniformly Lipschitz continuous in the decision variables, and the set of feasible decision variables is bounded, the cutting plane approach developed by Mutapcic and Boyd (2009) \cite{mutapcic2009cutting} converges to an optimal solution. Furthermore, when the function $\dsp g$ is concave in the uncertain parameters, then each iteration can be efficiently solved.

We highlight some solution methods for the aforementioned loss functions in Table \ref{solution_classification}.

\begin{table}[H]
\caption{Solution method for classification-based uncertainty sets.} \label{solution_classification}
\begin{center}
\begin{tabular}{||c|c|c||} 
 \hline
 Loss name & Loss function & Solution Method \\
 \hline\hline
 Cross-entropy & $\dsp - \sum_{j=1}^{l} \log \left( \Hat{y}_{i,j} \right) y_{i,j}$ & Linear in $\dsp \bm{y}$, cutting planes \\ 
 \hline
 Hinge Loss & $\dsp \max \left( 0, 1 - \left(2y_{i}-1\right) \left(2 \Hat{y}_{i}-1\right) \right)$ & \makecell{Piecewise linear convex \\ function in $\dsp \bm{y}$, cutting planes} \\
 \hline
 Misclassification Loss & $\dsp \bm{1} \left( \Hat{\bm{y}}_{i} \neq \bm{y}_{i} \right)$ & Polyhedral set or cutting planes \\
 \hline
\end{tabular}
\end{center}
\end{table}

We can draw parallels between these losses and uncertainty sets from the distributionally robust literature (Delage and Ye (2010) \cite{delage2010distributionally})---specifically, the uncertainty sets based on the $\dsp \phi$-divergence proposed by Ben-Tal et al. (2013) \cite{ben2013robust}. These uncertainty sets leverage measures of distance between an estimated probability distribution and the true probability distribution. In our classification framework, we use the distance between the predicted probability distribution $\dsp \Hat{\bm{y}}_{i}$ and the probability distribution of one of these events occurring, i.e., a binary $\dsp \bm{y}_{i}$.

\begin{proposition} \label{dro_KLB}
    When the loss function $\dsp \ell$ is the cross-entropy loss function $\dsp \Hat{\bm{y}}_{i} \mapsto - \sum_{j=1}^{l} \log \left( \Hat{y}_{i,j} \right) y_{i,j}$ where $\dsp l$ is the number of classes, the uncertainty set (\ref{uncertainty_set}) coincides with the $\dsp \phi$-divergence uncertainty set based on the Kullback-Leibler divergence proposed by Ben-Tal et al. (2013) (\cite{ben2013robust}).
\end{proposition}

\begin{proof}{Proof}
    Since for all $\dsp j \in [l]$ we have $\dsp y_{i,j} \log{y_{i,j}} = 0$ when $\dsp y_{i,j}$ is binary (i.e. both when $\dsp y_{i,j} = 0$ or $\dsp y_{i,j} = 1$). This means we can rewrite:
    \begin{align*}
        \dsp - \sum_{j=1}^{l} \log \left( \Hat{y}_{i,j} \right) y_{i,j}
        = &\sum_{j=1}^{l} \log \left( y_{i,j} \right) y_{i,j} -\sum_{j=1}^{l} \log \left( \Hat{y}_{i,j} \right) y_{i,j} \\
        = &\sum_{j=1}^{l} \log \left( \frac{y_{i,j}}{\Hat{y}_{i,j}} \right) y_{i,j}.
    \end{align*}
    
    This is exactly the $\dsp \phi$-divergence based on the Kullback-Leibler divergence.
\end{proof}

\begin{proposition}
    When the loss function $\dsp \ell$ is the hinge loss function $\dsp \Hat{y}_{i} \mapsto \max \left( 0, 1 - \left(2y_{i}-1\right) \left(2 \Hat{y}_{i}-1\right) \right)$ with inputs $\dsp \Hat{y}_{i} \in \left[ 0, 1\right]$, the uncertainty set (\ref{uncertainty_set}) coincides with the $\dsp \phi$-divergence uncertainty set based on the variation distance divergence proposed by Ben-Tal et al. (2013) \cite{ben2013robust} up to a constant scaling.
\end{proposition}

\begin{proof}{Proof}
    Since $\dsp y_{i}$ is binary and $\dsp \Hat{y}_{i}$ is a probability:
    \begin{equation*}
    \dsp \max \left( 0, 1 - \left(2y_{i}-1\right) \left(2 \Hat{y}_{i}-1\right) \right) = 2 \begin{cases}
        \max \left( 0, \Hat{y}_{i} - y_{i} \right), \text{ if } y_{i} = 0 \\
        \max \left( 0, y_{i} - \Hat{y}_{i} \right), \text{ if } y_{i} = 1
    \end{cases}
    = 2 \left| \Hat{y}_{i} - y_{i} \right|.
\end{equation*}

This is exactly the $\dsp \phi$-divergence based on the variation distance divergence.
\end{proof}

Finally, the misclassification loss yields the uncertainty set used in Bertsimas et al. (2019) \cite{bertsimas2019robust}, which incorporates the uncertainty around potentially corrupted labels. The authors reformulate the problem tractably---indeed, when the function $\dsp g$ is concave and the radius of the uncertainty is integer, the integrality constraints can be relaxed.

\medbreak
\noindent \textbf{Regression}
\medbreak

\noindent The mean squared error, the mean absolute error, and the Huber loss functions are among the most utilized in the machine learning literature for regression tasks. These losses are the basis of a wide range of machine learning models from linear regression, neural networks, to Optimal Regression Trees.

Solving constraint (\ref{constraint}) with the uncertainty set (\ref{uncertainty_set}) in the regression case can be efficiently tackled either through a cutting planes approach, or by deriving the robust counterpart. Indeed, Ben-Tal et al. (2015) \cite{ben2015deriving}, show that when the uncertainty set is nonempty, convex, and compact (which happens as soon as $\dsp \Hat{\bm{y}} \in ri \left( dom \left( g \left( . , \bm{x} \right) \right) \right), \forall \bm{x}$, and $\dsp \rho > 0$ for the aforementioned choices) that constraint (\ref{constraint}) utilizing the uncertainty set (\ref{uncertainty_set}) is verified if and only if there exists $\dsp \bm{v} \in \bbR^{m}$ and $\dsp u \geq 0$ such that:
\begin{equation}
    \label{inequality0}
    \dsp \rho \, u \, \ell^{*} \left( \frac{\bm{v}}{u}, \Hat{\bm{y}} \right) - g_{*} \left( \bm{v}, \bm{x} \right) \leq 0,
\end{equation}

\noindent where $\dsp g_{*}$ is the concave conjugate of $\dsp g$ (with respect to the first variable) and $\dsp \ell^{*}$ is the convex conjugate of $\dsp \ell$ (with respect to the first variable). Recall that when $\dsp \ell = \sum_{j} \ell_{j}$, then (\ref{inequality0}) is equivalent to the existence of $\dsp \bm{v} \in \bbR^{m}$ and $\dsp u \geq 0$ such that:
\begin{equation}
    \rho \, u \, \sum_{j} \ell_{j}^{*} \left( \frac{v_{j}}{u}, \Hat{y}_{j} \right) - g_{*} \left( \bm{v}, \bm{x} \right) \leq 0.
\end{equation}

\noindent We provide some dual functions of the aforementioned loss functions in Table \ref{solution_regression}.

\begin{table}[H]
\caption{Dual functions of regression losses.} \label{solution_regression}
\begin{center}
\resizebox{\textwidth}{!}{
\begin{tabular}{||c|c|c||} 
 \hline
 Loss Name & Loss function & Dual Function \\
 \hline\hline
Squared Error & $\dsp \left( y_{i} - \Hat{y}_{i} \right)^{2}$ & $\dsp \ell_{i}^{*} \left( y_{i}, \Hat{y}_{i} \right) = \Hat{y}_{i} y_{i} + \frac{y_{i}^2}{4}$ \\
 \hline
Absolute Error & $\dsp \left| y_{i} - \Hat{y}_{i} \right|$ & $\dsp \ell_{i}^{*} \left( y_{i}, \Hat{y}_{i} \right) = \begin{cases}
     0, \text{ if } \left|y_{i} \right| \leq 1, \\
     + \infty, \text{ if } \left|y_{i} \right| > 1
 \end{cases}$ \\
 \hline
  Huber Loss & $\dsp \begin{cases} \frac{1}{2} \left( y_{i} - \Hat{y}_{i} \right)^2, \text{ if } \left| y_{i} - \Hat{y}_{i} \right| \leq \delta, \\ \delta \left| y_{i} - \Hat{y}_{i} \right| - \frac{1}{2} \delta^2, \text{ if } \left| y_{i} - \Hat{y}_{i} \right| > \delta \end{cases}$ & $\dsp \ell_{i}^{*} \left( y_{i}, \Hat{y}_{i} \right) = \begin{cases} \frac{1}{2} \left( y_{i} - \Hat{y}_{i} \right)^2, \text{ if } \left| y_{i} - \Hat{y}_{i} \right| \leq \delta, \\ + \infty, \text{ if } \left| y_{i} - \Hat{y}_{i} \right| > \delta \end{cases}$\\
  \hline
\end{tabular}}
\end{center}
\end{table}

Notice that if all uncertain parameters are predicted using the mean squared error (respectively, the mean absolute error) loss function, then the uncertainty set (\ref{uncertainty_set}) coincides with the ellipsoidal uncertainty set with radius $\dsp \rho^{2}$ (respectively, the $\dsp L_{1}$ ball with radius $\dsp \rho$) with unit variance.

Unlike in the classification setting where the output of a model can often be interpreted as the probability distribution of each class occurring---from which the prediction's variance is therefore captured---the regression setting does not capture this heteroscedasticity.

\subsection{Taking Heteroscedasticity into Account for Regression} \label{heteroscedastic_regression}

In this section, we take into account the heteroscedasticity of the prediction of $\dsp \bm{y}$ given $\dsp \bm{X}$. To achieve this, we combine ideas of classical robust optimization and neural networks' variance estimation (Nix and Weigend (1994) \cite{nix1994estimating}, Bishop (1994) \cite{bishop1994mixture}, and Skafte et al. (2019) \cite{skafte2019reliable}) to develop a tailored uncertainty set in the regression setting. This approach extends the popular ellipsoidal uncertainty set to the setting where the mean and variance of uncertain parameters are estimated.

We train a model, such as a neural network, with two outputs: $\dsp \Hat{\bm{y}} \left( \bm{X} \right)$, representing the predicted value, and $\dsp \Hat{\bm{\sigma}}^{2} \left( \bm{X} \right)$, the variance around that prediction with the loss function proposed by Nix and Weigend (1994) \cite{nix1994estimating}:
\begin{equation*}
    \dsp \ell \left( \bm{y}, \Hat{\bm{y}}, \Hat{\bm{\sigma}} \right) = \sum_{i=1}^{m} \frac{ \left(y_{i} - \Hat{y}_{i} \right)^{2}}{\Hat{\sigma}^{2}_{i}} + \log \left( \Hat{\sigma}_{i}^{2} \right).
\end{equation*}

Notice that unlike in Kendall et al. (2018) \cite{kendall2018multi}'s approach to weight multiple loss functions, the variance $\dsp \Hat{\bm{\sigma}}^{2} \left( \bm{X} \right)$ is a function of the input. The uncertainty set we therefore obtain is:
\begin{equation*}
    \dsp \mathcal{U} = \left\{ \bm{y} \, \middle| \, \sum_{i=1}^{m} \frac{ \left(y_{i} - \Hat{y}_{i} \right)^{2}}{\Hat{\sigma}^{2}_{i}} \leq \rho' \right\},
\end{equation*}

\noindent where $\dsp \rho' = \rho - \sum_{i=1}^{m} \log \left( \Hat{\sigma}_{i}^{2} \right)$. Notice this is nearly identical to the ellipsoidal uncertainty set in classical robust optimization, except we replace the estimated mean and variance with the predicted values of both quantities. Solving a constraint utilizing the uncertainty set based on regression with variance prediction is therefore as tractable as with the classical ellipsoidal uncertainty set.

Since the predictions of the value and of the variance minimize the loss independently from each other, one could train separate models for each output. Furthermore, although this approach has been developed in the context of neural networks, the loss functions can also be optimized for other machine learning models such as regression trees.

\section{Deriving Probabilistic Guarantees} \label{sec:probabilistic_guarantees}

A key strength of robust optimization is its ability to derive probabilistic guarantees on the violation of a constraint when utilizing specific uncertainty sets. We first obtain general guarantees by bounding the probability of the uncertain parameter belonging to the uncertainty set (an approach also used by Tulabandhula and Rudin (2014) \cite{tulabandhula2014robust} and Sun et al. (2023) \cite{sun2023predict}). These bounds are then improved, as suggested by Bertsimas et al. (2021) \cite{bertsimas2021probabilistic}, by directly bounding the probability of violation of the constraint, when using the mean squared error loss function (with or without predicting variance, as discussed in Section \ref{heteroscedastic_regression}).

\subsection{General Guarantees}

We first present general guarantees that apply to any loss function $\dsp \ell$ and any constraint function $\dsp g$.

\begin{theorem}
    \label{general_guarantees}
    Suppose we have access to independent, identically distributed values of the loss $\dsp \Tilde{Z} = \ell \left(  \Tilde{\bm{y}}, f \left( \Tilde{\bm{X}} \right) \right)$ that we will denote $S = \left\{ Z_{1}, 
    \dots, Z_{N} \right\}$. If a new realization $\dsp Z$ of $\dsp \Tilde{Z}$ is independent from $\dsp Z_{1}, 
    \dots, Z_{N}$, then, for any $\dsp i \in [N]$:
    \begin{equation} \label{general_guarantee_inequality}
        \dsp \mathbb{P}_{Z, S} \left( Z > Z_{\left( i \right)} \right) \leq 1 - \frac{i}{N+1},
    \end{equation}

    \noindent where $\dsp Z_{\left( i \right)}$ is the $\dsp i$th order statistic of the sample $
    \dsp S$. Note that the probability is with regards to $\dsp Z$ and the sample $\dsp S$.
\end{theorem}

Practically, if we have access to a validation set $\dsp S = \left( \bm{X}^{i}, \bm{y}^{i} \right)_{i \in [N]}$, we can calculate the statistic $\dsp Z$. Then, by setting the radius of (\ref{uncertainty_set}) to $\dsp \rho = Z_{\left( i \right)}$ with $\dsp i = \left \lceil \left( N + 1 \right) \left( 1 - \alpha \right) \right\rceil$ (i.e, setting $\dsp Z_{\left( i \right)}$ to be the empirical $\dsp \left( 1 + \frac{1}{N} \right) \left( 1 - \alpha \right)$-quantile of $\dsp \Tilde{Z}$), we obtain that any $\dsp \bm{x}$ that verifies constraint (\ref{constraint}) will violate $\dsp g \left( \bm{x}, \Tilde{\bm{y}} \right) \leq 0$ with probability at most $\dsp \alpha$.

\begin{proof}{Proof}
    Let $\dsp i \in [N]$. We denote by $\dsp F$ the cumulative distribution function of $\dsp \Tilde{Z}$ and $\dsp F^{-1}$ its quantile function. By the law of total expectation and the independence of $\dsp Z$ and $\dsp Z_{\left( i \right)}$:
    \begin{align}
        \dsp \mathbb{P} \left( Z \leq Z_{\left( i \right)} \right)
        = &\mathbb{E} \left[ \mathbb{P} \left( Z \leq Z_{\left( i \right)} \middle| Z_{\left( i \right)} \right) \right] \label{proof_general_guarantee:part1} \\
        = &\mathbb{E} \left[ F \left( Z_{\left( i \right)} \right) \right] \nonumber\\
        = &\int_{0}^{1} \mathbb{P} \left( F \left( Z_{\left( i \right)} \right) > t \right) dt, \nonumber
    \end{align}
    
    \noindent since $\dsp F \left( Z_{\left( i \right)} \right) \in \left[0,1\right]$. By the continuity of probability, we obtain that for $\dsp t \in \left[ 0, 1 \right]$:
    \begin{align}
        \dsp \mathbb{P} \left( F \left( Z_{\left( i \right)} \right) > t \right)
        = &\mathbb{P} \left( \bigcup_{m \geq 1} \left\{ F \left( Z_{\left( i \right)} \right) \geq t + \frac{1}{m}\right\} \right) \label{proof_general_guarantee:part2} \\
        = &\mathbb{P} \left( \bigcup_{m \geq 1} \left\{Z_{\left( i \right)} \geq F^{-1} \left( t + \frac{1}{m} \right) \right\} \right) \nonumber \\
        = & \lim_{m \to +\infty} \mathbb{P} \left(Z_{\left( i \right)} \geq F^{-1} \left( t + \frac{1}{m} \right) \right). \nonumber
    \end{align}

     Calculating the probability $\dsp \mathbb{P} \left(Z_{\left( i \right)} \geq F^{-1} \left( t + \frac{1}{m} \right) \right)$ using properties of order statistics, we obtain that it is equal to the probability that a binomial distribution with probability $\dsp \mathbb{P} \left( Z < F^{-1} \left( t + \frac{1}{m} \right) \right)$ and $\dsp N$ trials has at most $\dsp i-1$ successes. Since the cumulative distribution function of a binomial distribution is a decreasing function of its probability of success, we can obtain a lower bound of (\ref{proof_general_guarantee:part2}) by upper bounding $\dsp \mathbb{P} \left( Z < F^{-1} \left( t + \frac{1}{m} \right) \right)$.

     Let $\dsp m \geq 1$ and $t \in \left[ 0, 1 \right]$. We have that for all $\dsp k \geq 1$:
     \begin{equation*}
         \dsp F \left( F^{-1} \left( t + \frac{1}{m} \right) - \frac{1}{k} \right) < t + \frac{1}{m}.
     \end{equation*}
     
     Indeed, by contradiction, if $\dsp F \left( F^{-1} \left( t + \frac{1}{m} \right) - \frac{1}{k} \right) \geq t + \frac{1}{m}$, then $\dsp F^{-1} \left( t + \frac{1}{m} \right) - \frac{1}{k} \geq \dsp F^{-1} \left( t + \frac{1}{m} \right)$, which is a contradiction. Therefore:
     \begin{align*}
         \dsp &\lim_{k \to +\infty} \mathbb{P} \left( Z \leq F^{-1} \left( t + \frac{1}{m} \right) - \frac{1}{k} \right) \leq t + \frac{1}{m}, \\
         \text{so } &\mathbb{P} \left( \bigcup_{k \geq 1} \left\{ Z \leq F^{-1} \left( t + \frac{1}{m} \right) - \frac{1}{k} \right\}\right) \leq t + \frac{1}{m}, \\
         \text{so } &\mathbb{P} \left( Z < F^{-1} \left( t + \frac{1}{m} \right) \right) \leq t + \frac{1}{m},
     \end{align*}

    \noindent by using the continuity of probability. We thus lower bound (\ref{proof_general_guarantee:part2}) as follows:
    \begin{align*}
         \dsp \mathbb{P} \left( F \left( Z_{\left( i \right)} \right) > t \right)
        \geq &\lim_{m \to +\infty} \sum_{j=0}^{i-1} \binom{N}{j} \left( t + \frac{1}{m} \right)^{j} \left(1 - t - \frac{1}{m} \right)^{N-j} \\
        \geq &\sum_{j=0}^{i-1} \binom{N}{j} t^{j} \left( 1 - t \right)^{N-j}.
     \end{align*}

     Leveraging (\ref{proof_general_guarantee:part1}), we thus obtain the following bound:
     \begin{align*}
        \dsp \mathbb{P} \left( Z \leq Z_{\left( i \right)} \right)
        \geq &\sum_{j=0}^{i-1} \binom{N}{j} \int_{0}^{1} t^{j} \left( 1 - t \right)^{N-j} dt \\
        \geq & \sum_{j=0}^{i-1} \frac{1}{N+1} \\
        \geq & \frac{i}{N+1},
    \end{align*}

    \noindent which means that $\dsp \mathbb{P} \left( Z > Z_{\left( i \right)} \right) \leq 1 - \frac{i}{N+1}$.
\end{proof}

Note that inequality (\ref{general_guarantee_inequality}) is tight since the statement becomes an equality as soon as $\dsp \Tilde{Z}$ is continuous.

Furthermore, notice that since the training of the machine learning model $\dsp f$ aims to minimize $\dsp Z$, the radius of the uncertainty set $\dsp Z_{\left( i \right)}$ should shrink as the model's errors shrink. This means that the higher quality the model, the smaller the necessary radius of the uncertainty set to obtain a desired threshold probability of violation.

\subsection{Guarantees Tailored for Regression with Mean Squared Error Loss}

The general bounds we proposed in Theorem \ref{general_guarantees} can be improved upon by taking into account the structure of the constraint function $\dsp g$, as argued by Bertsimas et al. (2021) \cite{bertsimas2021probabilistic}. The authors prove guarantees in the case where the uncertain parameters have independent components $\dsp \Tilde{\bm{y}}$ with known mean and variance. We then extend their findings to the more general setting of our machine learning based ellipsoidal uncertainty set, with predicted mean and variance estimates.

\begin{theorem} \label{strong_guarantees}
    Suppose that the components of $\dsp \left. \Tilde{\bm{y}} \middle| \Tilde{\bm{X}} \right.$ are sub-Gaussian and independent. Furthermore, suppose that for all $\dsp \bm{X} \in \mathcal{X}, \, \Hat{\bm{y}} \left( \bm{X} \right) \in ri \left( dom \left( g \left( . , \bm{x} \right) \right) \right), \, \forall \bm{x}$. Let $\dsp \bm{x}$ verify the robust constraint:
    \begin{equation}
        \label{concave_constraint}
        \dsp g \left( \bm{y}, \bm{x} \right) \leq 0, \quad \forall \bm{y} \in \mathcal{U} \left( \bm{X} \right) = \left\{ \bm{y} \middle| \left\lVert \frac{\bm{y} - \Hat{\bm{y}} \left( \bm{X} \right) }{\Hat{\bm{\sigma}}\left( \bm{X} \right)} \right\rVert_{2} \leq \rho \right\},
    \end{equation}

    \noindent where $\dsp g$ is concave and $\dsp \rho > 0$. Then, we have the following guarantee:
    \begin{equation} \label{strong_guarantees_equation}
        \dsp \mathbb{P} \left( g \left( \Tilde{\bm{y}}, \bm{x} \right) > 0 \right)
        \leq \mathbb{E} \left[ \exp \left( - \frac{1}{2} \left(\frac{\rho^{2}}{\Tilde{t}_{m}} - m \right) \right) \right] + \exp \left( \frac{m}{2} \right) \epsilon \left( \frac{\rho^{2}}{\Tilde{t}_{m}} \right),
    \end{equation}
    
    \noindent where $\dsp \Tilde{t}_{m} = \max_{j \in [m]} \Tilde{L}_{j}$, $\dsp \Tilde{L}_{j} = \left( \frac{\Tilde{y}_{j} - \Hat{y}_{j}}{\Hat{\sigma}_{j}} \right)^{2}$
    and $\dsp \epsilon \left( \Tilde{u} \right) = \mathbb{E} \left[ q \left(  \Tilde{u} \right) \right] - q \left( \mathbb{E} \left[ \Tilde{u} \right] \right) $ is the Jensen gap of the function $\dsp q : t \mapsto \begin{cases} 0, \text{ if } t \geq 2, \\ \exp \left( -\frac{1}{2} \right) \left(\frac{1}{t}-\frac{1}{2}\right) + \exp \left( -2 \right) - \exp \left( - t \right), \text{ if } t < 2.\end{cases}$ In the limit (in $\dsp m$) where the weak Law of Large Number applies, then this bound can be improved into the following:
    \begin{equation} \label{strong_guarantees_LLN}
        \dsp \mathbb{P} \left( g \left( \Tilde{\bm{y}}, \bm{x} \right) > 0 \right)
        \leq \mathbb{E} \left[ \exp \left( - \frac{1}{2} \left(\frac{\rho^{2} - L}{\Tilde{t}_{m}} \right) \right) \right] + \epsilon \left( \frac{\rho^{2} - L}{\Tilde{t}_{m}} \right),
    \end{equation}

    \noindent where $\dsp \rho^{2} \geq L = \sum_{i=1}^{m} \mathbb{E} \left[ \Tilde{L}_{j} \right]$.
\end{theorem}

Notice that this bound mimics that of the celebrated guarantees of the ellipsoidal uncertainty set, which decays at a rate of $\dsp e^{-\frac{\rho^{2}}{2}}$. There are three
 differences:
\begin{enumerate}
    \item The radius of the uncertainty set is corrected by $\dsp m$, (respectively, $\dsp \frac{L}{\Tilde{t}_{m}}$) in (\ref{strong_guarantees_equation}) (respectively, (\ref{strong_guarantees_LLN})), to take into account the imprecision of our estimates $\dsp \Hat{\bm{y}} \left( \bm{X} \right)$ compared to the true $\dsp \Bar{\bm{y}} \left( \bm{X} \right) = \mathbb{E} \left[ \Tilde{\bm{y}} \middle| \Tilde{\bm{X}} = \bm{X} \right]$.
    \item The radius of the uncertainty set is scaled by $\dsp \frac{1}{\Tilde{t}_{m}}$, meaning the bound is as good as its worst model. Indeed, an inaccurate prediction of a single component could vastly change the optimal strategy $\dsp \bm{x}$. This term usually scales logarithmically as a function of $\dsp m$. For example, in the case where the error term $\dsp \frac{\Tilde{y}_{j} - \Hat{y}_{j}}{\Hat{\sigma}_{j}}$ is a standard normal distribution (the error term being normally distributed is a common hypothesis when the model is well calibrated), then $\dsp \frac{\Tilde{t}_{m}}{\log \left(m \right)}$ converges to a constant as $\dsp m$ grows.
    \item The additive corrective term $\dsp \epsilon$ estimates the Jensen gap of the function $\dsp q$. Conservatively, one can notice that by concavity of $t \mapsto \dsp q \left( \frac{1}{t} \right) - \beta t^{2}$ (where $\dsp \beta = \frac{\exp \left( - 1 / \alpha \right) }{\alpha^{4}} \left( 2 \alpha - 1 \right) \approx 0.42$ with $\dsp \alpha = \frac{3+\sqrt{3}}{6}$), Jensen's inequality yields that $\dsp \epsilon \left( \frac{\rho^{2} - L}{\Tilde{t}_{m}} \right) \leq \frac{\beta}{\left(\rho^2 - L\right)^{2}}  Var \left( \Tilde{t}_{m}^{2} \right)$. However, this upper bound is conservative and the value should shrink as $\dsp \rho$ is large since the random variable $\dsp \Tilde{t}_{m}^{2}$ is concentrated on the linear portion of $\dsp q$.
\end{enumerate}

\begin{proof}{Proof}
    Since constraint (\ref{concave_constraint}) holds,  $\dsp \Hat{\bm{y}} \in ri \left( dom \left( g \left( . , \bm{x} \right) \right) \right), \forall \bm{x}$, and $\dsp \rho > 0$ making the uncertainty set nonempty, convex, and compact, we obtain that there exists $\dsp \bm{v} \in \bbR^{m}$ such that:
    \begin{equation}
        \label{inequality1}
        \dsp \Hat{\bm{y}}^{\top} \bm{v} + \rho \left\lVert \Hat{\bm{\sigma}} \cdot \bm{v} \right\rVert_{2} - g_{*} \left( \bm{v}, \bm{x} \right) \leq 0,
    \end{equation}

    \noindent where $\dsp g_{*}$ is the concave conjugate of $\dsp g$ (with respect to the first variable), as proven by Ben-Tal et al. (2015) \cite{ben2015deriving}. Given  inequality (\ref{inequality1}), Bertsimas et al. (2021) \cite{bertsimas2021probabilistic}, prove that sub-Gaussian random variables $\dsp \left. \Tilde{\bm{y}} \middle| \Tilde{\bm{X}} \right.$ verify:
    \begin{equation}
        \label{main_inequality}
        \dsp \mathbb{P} \left( g \left( \Tilde{\bm{y}}, \bm{x} \right) > 0 \, \middle| \, \Tilde{\bm{X}} \right)
        \leq \exp \left( - \frac{1}{2} \left(\frac{ \rho \left\lVert \bm{\Hat{\sigma}} \cdot \bm{v} \right\rVert_{2} + \left(\Hat{\bm{y}} - \Bar{\bm{y}}\right)^{\top} \bm{v}}{\left\lVert \bm{\sigma} \cdot \bm{v} \right\rVert_{2} } \right)^{2} \right),
    \end{equation}

    \noindent where $\dsp \Bar{\bm{y}} \left( \bm{X} \right) = \mathbb{E} \left[ \Tilde{\bm{y}} \middle| \Tilde{\bm{X}} = \bm{X} \right]$ and $\dsp \sigma_{i} \left( \bm{X} \right)$ is the proxy variance of $\dsp \left. \Tilde{y}_{i} \middle| \Tilde{\bm{X}} \right.$ for $\dsp i \in [m]$.

    \noindent Notice that for all $\dsp i \in [m]$:
    \begin{equation*}
        \dsp \left( \frac{\Hat{y}_{i} - \Bar{y}_{i}}{\sigma_{i}} \right)^{2} = \left( \frac{\Hat{\sigma}_{i}}{\sigma_{i}} \right)^{2} \mathbb{E} \left[ \left( \frac{\Tilde{y}_{i} - \Hat{y}_{i}}{\Hat{\sigma}_{i}} \right)^{2} \middle| \Tilde{\bm{X}} \right] - 1.
    \end{equation*}

    \noindent We can prove that:
    \begin{align} \label{proof_tailored_guarantee:part1}
        \forall \bm{\alpha} \in \mathcal{C}, \ &\rho \sqrt{\sum_{i=1}^{m} \left( \frac{\Hat{\sigma}_{i}}{\sigma_{i}} \right)^{2} \alpha_{i}^{2}}  - \sum_{i=1}^{m} \alpha_{i} \sqrt{\left( \frac{\Hat{\sigma}_{i}}{\sigma_{i}} \right)^{2} \mathbb{E} \left[ \left( \frac{\Tilde{y}_{i} - \Hat{y}_{i}}{\Hat{\sigma}_{i}} \right)^{2} \middle| \Tilde{\bm{X}} \right] - 1} \\
        \geq &\sqrt{\rho^{2} - \sum_{i=1}^{m} \mathbb{E} \left[ \left( \frac{\Tilde{y}_{i} - \Hat{y}_{i}}{\Hat{\sigma}_{i}} \right)^{2} \middle| \Tilde{\bm{X}} \right]} \sqrt{\sum_{i=1}^{m} \frac{\alpha_{i}^{2}}{\mathbb{E} \left[ \left( \frac{\Tilde{y}_{i} - \Hat{y}_{i}}{\Hat{\sigma}_{i}} \right)^{2} \middle| \Tilde{\bm{X}} \right]}} \nonumber \\
        \geq &\sqrt{\frac{\rho^{2} - \sum_{i=1}^{m} \mathbb{E} \left[ \left( \frac{\Tilde{y}_{i} - \Hat{y}_{i}}{\Hat{\sigma}_{i}} \right)^{2} \middle| \Tilde{\bm{X}} \right]}{\max_{j \in [m]} \mathbb{E} \left[ \left( \frac{\Tilde{y}_{j} - \Hat{y}_{j}}{\Hat{\sigma}_{j}} \right)^{2} \middle| \Tilde{\bm{X}} \right]}}, \nonumber
    \end{align}
    \noindent where $\mathcal{C}$ is the unit circle. The first inequality is a consequence of minimizing $h:\bm{u} \leq 0 \mapsto \dsp \rho \sqrt{\sum_{i=1}^{m} u_{i} \alpha_{i}^{2}}  - \sum_{i=1}^{m} \alpha_{i} \sqrt{ u_{i} \Tilde{L}_{i} - 1}$. If $\dsp \rho^{2} - \sum_{i=1}^{n} \Tilde{L}_{i} < 0$, then $\dsp h$ is unbounded below (take $\dsp \bm{v}$ such that $\dsp \forall i \in [m], v_{i} = \frac{\Tilde{L}_{i}}{\alpha_{i}^{2} \sum_{j=1}^{m}\Tilde{L}_{j}}t$ then $h \left( \bm{v} \right) \xrightarrow[t \to \infty]{} - \infty$, meaning the minimum value of $\dsp h^{2}$ is $\dsp 0$. Similarly, when $\dsp \rho^{2} - \sum_{i=1}^{n} \Tilde{L}_{i} = 0$), that same limit is $0$, meaning $\dsp h^{2}$ is lower bounded by 0. If  $\dsp \rho^{2} - \sum_{i=1}^{n} \Tilde{L}_{i} > 0$, calculating its partial derivatives $\dsp \frac{\partial h}{\partial u^{*}_{i}} \left( u^{*}_{i} \right) = \frac{\rho \alpha_{i}^{2}}{2\sqrt{\sum_{i=1}^{m} u^{*}_{i} \alpha_{i}^{2}}} + \frac{ \alpha_{i} \Tilde{L}_{i}}{2\sqrt{L_{i} u^{*}_{i} - 1}}$ and setting them to zero, we obtain that $\dsp \bm{L} \left( \rho^{2} I_{m} - \bm{L} \bm{e}\bm{e}^{\top} \right) \bm{\alpha}^{2}\bm{u^{*}} =\rho^{2} \bm{\alpha}^{2} \bm{e}$ so $\dsp \bm{u^{*}} = \bm{L}^{-1} \bm{e} + \frac{1}{\rho^{2} - \sum_{i=1}^{m} L_{i}} \bm{\alpha}^{-2} \bm{L} \bm{e} \bm{e}^{\top} \bm{L}^{-1} \bm{\alpha}^{2} \bm{e}$, where $\dsp \bm{e}$ is a vector of ones and $\dsp \bm{L}$ (respectively, $\dsp \bm{\alpha}$) is a diagonal matrix whose components are the $\dsp L_{i}$ (respectively, $\dsp \alpha_{i}$) for $\dsp i \in [m]$. Notice that $\dsp \bm{u}^{*}$ is in the domain of $\dsp h$. Since $\dsp \bm{u} \mapsto h \left( \bm{u}^{2} \right)$ is convex (as a sum of convex functions), this shows that $\dsp \bm{u}^{*}$ is the minimum of $\dsp h$.
    
    By applying the above inequality (\ref{proof_tailored_guarantee:part1}) for $\dsp \bm{\alpha} = \frac{1}{\left\lVert \bm{\sigma} \cdot \bm{v} \right\rVert_{2}} \bm{\sigma} \cdot \bm{v} \in \mathcal{C}$, we obtain that:
    \begin{equation*}
        \dsp \mathbb{P} \left( g \left( \Tilde{\bm{y}}, \bm{x} \right) > 0 \, \middle| \, \Tilde{\bm{X}} \right) \leq \exp{ \left( -\frac{1}{2} \left( \frac{\rho^{2} - \sum_{i=1}^{m} \mathbb{E} \left[ \left( \frac{\Tilde{y}_{i} - \Hat{y}_{i}}{\Hat{\sigma}_{i}} \right)^{2} \middle| \Tilde{\bm{X}} \right]}{\max_{j \in [m]} \mathbb{E} \left[ \left( \frac{\Tilde{y}_{j} - \Hat{y}_{j}}{\Hat{\sigma}_{j}} \right)^{2} \middle| \Tilde{\bm{X}} \right]} \right) \right) }.
    \end{equation*}

    At this stage, if the approximation of the law of large numbers holds, we can replace $\dsp \sum_{i=1}^{m} \mathbb{E} \left[ \left( \frac{\Tilde{y}_{i} - \Hat{y}_{i}}{\Hat{\sigma}_{i}} \right)^{2} \middle| \Tilde{\bm{X}} \right]$ with $\dsp L$ and obtain the following bound:
    \begin{equation*}
        \dsp \frac{\rho^{2} - \sum_{i=1}^{m} \mathbb{E} \left[ \left( \frac{\Tilde{y}_{i} - \Hat{y}_{i}}{\Hat{\sigma}_{i}} \right)^{2} \middle| \Tilde{\bm{X}} \right]}{\max_{j \in [m]} \mathbb{E} \left[ \left( \frac{\Tilde{y}_{j} - \Hat{y}_{j}}{\Hat{\sigma}_{j}} \right)^{2} \middle| \Tilde{\bm{X}} \right]}
        \geq \frac{\rho^{2} - L}{\max_{j \in [m]} \mathbb{E} \left[ \left( \frac{\Tilde{y}_{j} - \Hat{y}_{j}}{\Hat{\sigma}_{j}} \right)^{2} \middle| \Tilde{\bm{X}} \right]}
        \geq \frac{\rho^{2} - L}{\mathbb{E} \left[ \Tilde{t}_{m} \middle| \Tilde{\bm{X}} \right]}.
    \end{equation*}
    
    Otherwise, we can bound:
    \begin{equation*}
        \dsp \frac{\rho^{2} - \sum_{i=1}^{m} \mathbb{E} \left[ \left( \frac{\Tilde{y}_{i} - \Hat{y}_{i}}{\Hat{\sigma}_{i}} \right)^{2} \middle| \Tilde{\bm{X}} \right]}{\max_{j \in [m]} \mathbb{E} \left[ \left( \frac{\Tilde{y}_{j} - \Hat{y}_{j}}{\Hat{\sigma}_{j}} \right)^{2} \middle| \Tilde{\bm{X}} \right]}
        \geq \frac{\rho^{2}}{\max_{j \in [m]} \mathbb{E} \left[ \left( \frac{\Tilde{y}_{j} - \Hat{y}_{j}}{\Hat{\sigma}_{j}} \right)^{2} \middle| \Tilde{\bm{X}} \right]} - m
        \geq \frac{\rho^{2}}{\mathbb{E} \left[ \Tilde{t}_{m} \middle| \Tilde{\bm{X}} \right]} - m.
    \end{equation*}

    We focus on the second case, the first case can be treated equivalently. We obtain:
    \begin{equation*}
        \dsp \mathbb{P} \left( g \left( \Tilde{\bm{y}}, \bm{x} \right) > 0 \right)
        \leq \exp \left( - \frac{m}{2} \right) \mathbb{E} \left[ \exp \left( - \frac{\rho^{2}}{2 \mathbb{E} \left[ \Tilde{t}_{m} \middle| \Tilde{\bm{X}} \right]} \right) \right].
    \end{equation*}

    Let us define $\dsp p_{1} : t \mapsto \exp \left( -\frac{1}{t} \right)$ and $\dsp p_{2} : t \mapsto q \left( \frac{1}{t} \right)$. By construction, $\dsp p_{1} + p_{2}$ and $\dsp p_{2}$ are both convex. We can therefore apply Jensen's inequality to obtain the following bounds:
    \begin{align*}
        \dsp &\mathbb{E} \left[ p_{1} \left( \frac{2 \mathbb{E} \left[ \Tilde{t}_{m} \middle| \Tilde{\bm{X}} \right]}{\rho^{2}} \right) \right] \\
        \leq &\mathbb{E} \left[ \left( p_{1} + p_{2} \right) \left( \frac{2 \mathbb{E} \left[ \Tilde{t}_{m} \middle| \Tilde{\bm{X}} \right]}{\rho^{2}} \right) - p_{2} \left( \frac{2 \mathbb{E} \left[ \Tilde{t}_{m} \middle| \Tilde{\bm{X}} \right]}{\rho^{2}} \right) \right] \\
        \leq &\mathbb{E} \left[ \left( p_{1} + p_{2} \right) \left( \frac{2 \Tilde{t}_{m}}{\rho^{2}} \right) \right] - p_{2} \left(  \frac{2 \mathbb{E} \left[ \Tilde{t}_{m} \right]}{\rho^{2}} \right).
    \end{align*}

    Therefore,
    \begin{equation*}
        \dsp \mathbb{P} \left( g \left( \Tilde{\bm{y}}, \bm{x} \right) > 0 \right)
        \leq \exp \left( \frac{m}{2} \right) \left( \mathbb{E} \left[ \left( p + q \right) \left( \Tilde{t}_{m} \right) \right] - q \left( \mathbb{E} \left[ \Tilde{t}_{m} \right] \right) \right),
    \end{equation*}
    \noindent which concludes the proof.
\end{proof}

The hypothesis of sub-Gaussian distribution is satisfied as soon as the distribution is bounded. Therefore, as long as $\dsp \Tilde{\bm{y}}$ is bounded, so is $\dsp \Tilde{\bm{y}} \, | \, \Tilde{\bm{X}}$, and this hypothesis is satisfied. Moreover, the hypothesis that the components of $\dsp \Tilde{\bm{y}}$ are independent conditionally on $\dsp \Tilde{\bm{X}}$ is analogous to a modeler deciding to train separate models to predict each component of $\dsp \Tilde{\bm{y}}$, when they might reasonably believe that knowing the value of one component does not provide additional information about another.

Notice that Theorem \ref{strong_guarantees} can also be extended to the case of the mean squared error predictor, where the predicted variance can be considered a constant equal to 1. We will demonstrate in the experimental section that the guarantees from Theorem \ref{strong_guarantees}
are stronger than existing approaches.

% \begin{remark}
% In our setting, we can adapt Sun et al. (2023) \cite{sun2023predict}'s approach to obtain the following guarantees, when $\alpha \in \left[0,1\right]$. Suppose we have access to a validation set $\dsp \mathcal{D}$, then by choosing $\dsp \rho$ in (\ref{uncertainty_set}) such that $\dsp \mathbb{P}_{\mathcal{D}} \left( \ell \left( \Tilde{\bm{y}}, f^{\star} \left( \Tilde{\bm{X}} \right) \right) \leq \rho \right) \geq \alpha$ and let $\dsp \bm{x}$ verify the robust constraint:
% \begin{align}
%     \dsp g \left( \bm{y}, \bm{x} \right) \leq 0, \ \forall \bm{y} \in \mathcal{U} = \left\{ \bm{y} \, \middle| \, \ell \left( \bm{y}, f^{\star} \left( \bm{X} \right) \right) \leq \rho \right\}.
% \end{align}

% \noindent We obtain the following guarantees:
% \begin{align*}
%     \dsp \mathbb{P} \left( g \left( \Tilde{\bm{y}}, \bm{x} \right) > 0 \right)
%     \leq 1 - \alpha + \frac{1}{\left| \mathcal{D} \right|^{2}} + \frac{4 \log \left| \mathcal{D} \right|}{\sqrt{\left| \mathcal{D} \right|}}.
% \end{align*}

% \noindent We refer the reader to the proof provided by Sun et al. (2023) \cite{sun2023predict}. This bound is strongest when the modeler has access to a large pool of validation data $\dsp \mathcal{D}$.
% \end{remark}

\section{Computational Experiments} \label{sec:computational_experiments}

In this section, we present computational experiments with synthetic data to examine the quality of our machine learning based uncertainty sets for both classification and regression. We study three classical optimization problems: a newsvendor problem, a portfolio optimization, and a shortest path problem. We are able to demonstrate that our proposed uncertainty sets are not as conservative as classical approaches and provide stronger guarantees than existing methods.

The data were generated following a similar procedure as that described in the experiments of Elmachtoub and Grigas (2022) \cite{elmachtoub2022smart}---the specifics for the data generation are detailed in each experiment's section. All machine learning models were trained with one hidden layer with 5 neurons and ReLU activation functions. An early stopping criteria is triggered when the loss has not improved for 10 consecutive epochs over a validation set (a random sample of 30\% of the training set).

To contrast our machine learning based uncertainty sets with other methods, we report the objective value for a given guarantee level. Once the values of the uncertain parameters have been revealed, we also report the regret (the gap between a given strategy and the optimal strategy). Results are reported on an independent test.

\subsection{Newsvendor problem}

\medbreak
\noindent \textbf{Experimental Setup}
\medbreak

\noindent We first consider a newsvendor problem, as described by Shapiro (2021) \cite{shapiro2021lectures}, where the objective is to determine the quantity $\dsp x$ of a good to order, given a stochastic demand $\dsp d_{j}$, with a probability $\dsp p_{j}$ to occur, while minimizing the following cost: 
\begin{equation*}
    \dsp
    \min_{x \geq 0} \quad \sum_{j=1}^{k} p_{j} \left( cx + b \left[ d_{j} - x \right]_{+} + h \left[ x - d_{j} \right]_{+} \right).
\end{equation*}

Here, $\dsp c > 0$ is associated with the ordering cost per unit of good, $\dsp b > c$ is the additional cost per unit when demand is not fully met, and $\dsp h$ is the holding cost for excess goods.

Suppose the modeler has access to data of past demand, and covariates $\dsp \bm{X}$ to help predict the probability that each scenario occurs. A classical distributionally robust optimization approach might ignore these covariates, and would estimate the probability of each scenario to occur, protecting against the worst case estimation of the probability distribution, as considered by Ben-Tal et al. (2013) \cite{ben2013robust}. Our approach instead leverages the covariates $\dsp \bm{X}$ to build a machine learning model predicting the probability of each scenario of occurring, then constructing an uncertainty set protecting the uncertainty in the model's output.

We set the number of training data points $\dsp N = 2000$, the number of covariates $\dsp p = 10$, as well as the parameters from the optimization model $\dsp c = 1$, $\dsp b = 6$, and $\dsp h = 2$. We consider two potential scenarios: low demand $\dsp d_{1} = 1$ and high demand $\dsp d_{2} = 10$. The distribution of the uncertain scenarios were constructed as follows: first, by generating an independent, standard normal vector of covariates $\dsp \bm{X} \in \bbR^{N \times p}$ as well as a matrix $\dsp \bm{B} \in \mathbb{R}^{p}$ where each entry follows a Bernoulli distribution that equals 1 with a probability of 0.5; next, by generating an additive noise level $\dsp \bm{\varepsilon} \in \bbR^{N}$ that is normally distributed; finally, by computing $\dsp \bm{q} = \frac{1}{\sqrt{p}} \left( \bm{X} \bm{B} \right) + \Bar{\varepsilon} \, \bm{\varepsilon}$, the low demand scenario occurs if $\dsp q_{i} \geq 0$, and the high demand scenario occurs if $\dsp q_{i} < 0$, for every data point $\dsp i \in [N]$, where  $\dsp \Bar{\varepsilon}$ is the noise level. We report results on an independent test set with  $\dsp 10,000$ samples

We vary the noise level $\dsp \Bar{\varepsilon} \in \left\{ 0, \, 0.1, \, 0.2, \, 0.5, \, 1 \right\}$ and the desired probabilistic guarantee at level $\dsp \alpha \in \left\{ 0.1, \, 0.05, \, 0.01 \right\}$.

We train a machine learning model, which predicts the probability that each scenario occurs, with the cross-entropy loss function. Motivated by Proposition \ref{dro_KLB}, we contrast our approach with that of the Kullback-Leibler $\dsp \phi$-divergence uncertainty set from Ben-Tal et al. (2013) \cite{ben2013robust}. We compare our general guarantees of Theorem \ref{general_guarantees} applied to this classification problem with their guarantees, using the values of the ``correction parameters'' $\dsp \delta_{\phi}$ and $\dsp \rho_{\phi}$ as defined on page 190 of Pardo (2018) \cite{pardo2018statistical}.
% Although the meaning of probabilistic guarantees in both works is different: Ben-Tal et al. (2013) \cite{ben2013robust}'s considering the probability that the true distribution lies in their uncertainty set, while ours considering the probability that the actual realization of the scenario is taken into account, the comparison

\medbreak
\noindent \textbf{Results}
\medbreak

\noindent We first report in Table \ref{results_newsvendor_RIG} the Relative Information Gain (RIG) as defined by He et al. (2014) \cite{he2014practical} for binary classification. The closer the value is to $\dsp 100\%$, the more information the machine learning model distills over a naive uniform prior. Since the $\dsp \phi$-divergence only has access to the uniform prior, this metric captures how much more knowledge our ML-based proposal leverages above the $\dsp \phi$-divergence approach.

\begin{table}[H]
\caption{Relative Information Gain (RIG) of machine learning model trained on dataset with varying noise level $\dsp \Bar{\varepsilon}$.}
\label{results_newsvendor_RIG}
\begin{center}
% \resizebox{\textwidth}{!}{
\begin{tabular}{|| c | c ||} 
 \hline
$\Bar{\varepsilon}$ & RIG \\
 \hline
 \hline\hline
 $0$ & $95\%$ \\ 
 \hline
 $0.1$ & $83\%$ \\ 
 \hline
 $0.2$ & $72\%$ \\ 
 \hline
 $0.5$ & $43\%$ \\ 
 \hline
 $1$ & $20\%$ \\ 
 \hline
\end{tabular}
\end{center}
\end{table}

We then compare the objective value as well as the regret achieved by each method as a function of the noise in the data $\dsp \Bar{\varepsilon}$ and the guarantee level $\dsp \alpha$. We report the trade-offs in objectives and guarantee levels in Table \ref{results_newsvendor_objectives}, and the trade-offs between regret and guarantees in Table \ref{results_newsvendor_regret}.

\begin{table}[H]
\centering
\caption{Average objective for a given guarantee level $\dsp \alpha$ and noise level $\dsp \Bar{\varepsilon}$ for different uncertainty sets.} \label{results_newsvendor_objectives}
\subfloat[$\dsp \phi$-divergence uncertainty set]{\begin{tabular}{|| c || c | c | c | c | c ||} 
 \hline
$\dsp \alpha$ $\diagup$ $\Bar{\varepsilon}$ & $0$ & $0.1$ & $0.2$ & $0.5$ & $1$ \\
 \hline
 \hline\hline
 $0.1$ & $27.9$ & $28.0$ & $28.0$ & $28.0$ & $28.0$ \\ 
 \hline
 $0.05$ & $28.0$ & $28.0$ & $28.0$ & $28.0$ & $28.0$ \\
 \hline
 $0.01$ & $28.0$ & $28.0$ & $28.0$ & $28.0$ & $28.0$ \\
 \hline
\end{tabular}}
\quad
\subfloat[Machine learning uncertainty set]{\begin{tabular}{|| c || c | c | c | c | c ||} 
 \hline
$\dsp \alpha$ $\diagup$ $\Bar{\varepsilon}$ & $0$ & $0.1$ & $0.2$ & $0.5$ & $1$ \\
 \hline
 \hline\hline
 $0.1$ & $5.4$ & $5.4$ & $5.5$ & $9.7$ & $15.8$ \\ 
 \hline
 $0.05$ & $5.4$ & $5.6$ & $7.6$ & $14.2$ & $20.0$ \\
 \hline
 $0.01$ & $5.9$ & $8.2$ & $11.7$ & $19.6$ & $25.6$ \\
 \hline
\end{tabular}}
\end{table}

\begin{table}[H]
\centering
\caption{Average regret for a given guarantee level $\dsp \alpha$ and noise level $\dsp \Bar{\varepsilon}$ for different uncertainty sets.} \label{results_newsvendor_regret}
\subfloat[$\dsp \phi$-divergence uncertainty set]{\begin{tabular}{|| c || c | c | c | c | c ||} 
 \hline
$\dsp \alpha$ $\diagup$ $\Bar{\varepsilon}$ & $0$ & $0.1$ & $0.2$ & $0.5$ & $1$ \\
 \hline
 \hline\hline
 $0.1$ & $22.1$ & $22.6$ & $22.6$ & $22.6$ & $22.6$ \\ 
 \hline
 $0.05$ & $22.6$ & $22.6$ & $22.6$ & $22.6$ & $22.6$ \\
 \hline
 $0.01$ & $22.6$ & $22.6$ & $22.6$ & $22.6$ & $22.6$ \\
 \hline
\end{tabular}}
\quad
\subfloat[Machine learning uncertainty set]{\begin{tabular}{|| c || c | c | c | c | c ||} 
 \hline
$\dsp \alpha$ $\diagup$ $\Bar{\varepsilon}$ & $0$ & $0.1$ & $0.2$ & $0.5$ & $1$ \\
 \hline
 \hline\hline
 $0.1$ & $0.4$ & $2.2$ & $4.1$ & $9.8$ & $15.9$ \\ 
 \hline
 $0.05$ & $0.4$ & $2.2$ & $4.5$ & $11.6$ & $17.4$ \\
 \hline
 $0.01$ & $0.6$ & $3.2$ & $6.8$ & $15.1$ & $20.6$ \\
 \hline
\end{tabular}}
\end{table}

Since there are plenty of data points for the radius of the $\dsp \phi$-divergence uncertainty set to be nearly $0$ for all noise levels and guarantee levels, its average objective and regret are nearly constant. Notice that for all noise and guarantee levels, our methods achieve lower values for both the objective and regret. Furthermore, as expected, as the noise level decreases, our RIG increases, meaning our models provide increasingly more information over a uniform prior. This in turn causes the objectives and regret to decrease, suggesting that our approach is most advantageous when machine learning predictions provide high signal.

\subsection{Portfolio Optimization} \label{portfolio_optimization}

\medbreak
\noindent \textbf{Experimental Setup}
\medbreak

\noindent We consider a classical portfolio optimization problem, where the objective is to maximize the returns:
\begin{align*}
    \dsp
    \min_{\bm{x} \in \bbR^{n}} \quad &-\sum_{i=1}^{n} r_{i} x_{i}, \\
    \text{s.t.} \quad &\sum_{i=1}^{n} x_{i} = 1, \\
    &\bm{x} \geq 0,
\end{align*}

\noindent where $\dsp \bm{x}$ represents the allocation of one unit of wealth between assets with return $\bm{r}$. In many applications, we do not have access to the knowledge of the exact returns $\dsp \bm{r}$, but we might have access to covariates $\dsp \bm{X}$ that help predict the returns.

The uncertain parameters' data were generated as follows: first, by generating an independent, uniformly distributed vector of covariates $\dsp \bm{X} \in \bbR^{N \times p}$ in $\dsp \left[-1, 1 \right]$ as well as a matrix $\dsp \bm{B} \in \mathbb{R}^{p}$ where each entry follows a Bernoulli distribution that equals $\dsp 1$ with a probability of $\dsp 0.5$; next, by sampling uniformly a multiplicative noise level $\dsp \bm{\varepsilon} \in \bbR^{N}$ in $\left[ 1 -
\Bar{\varepsilon}, 1 +
\Bar{\varepsilon} \right]$, where $\dsp \Bar{\varepsilon} \geq 0$; finally, by building the vector of uncertain parameters $\dsp \bm{y} = \frac{1}{\sqrt{p}} \left( \bm{X} \bm{B} \right) \cdot \bm{\varepsilon}$, which we then standardize. Notice that $\dsp \bm{\varepsilon}$ introduces some heteroscedastic noise.

For our experiments, we set the number of assets $\dsp n = 5$ and the number of covariates $\dsp p = 10$. We report results on an independent test set with  $\dsp 100$ samples. We experimented with various levels of noise $\dsp \Bar{\varepsilon} \in \left\{ 0, \, 0.1, \, 1 \right\}$ and training data points $\dsp N \in \left\{ 1000, \, 5000, \, 10000 \right\}$. We focused on probabilistic guarantees at level $\dsp \alpha \in \left\{ 0.1, \, 0.05, \, 0.01 \right\}$, using inequality (\ref{strong_guarantees_LLN}) of Theorem \ref{strong_guarantees}.

We build our uncertainty sets (\ref{uncertainty_set})  with two different loss functions $\dsp \ell$: the Mean Squared Error (MSE) and Mean Squared Error with Variance prediction (MSEV), as detailed in Section \ref{heteroscedastic_regression}. We compare our approaches against: (1) the robust ellipsoidal method that ignores the availability of covariates (Classical ellipsoidal), (2) the Smart Predict then Optimize approach from Elmachtoub and Grigas (2022) \cite{elmachtoub2022smart} (SPO), (3) the kNN approach proposed by Ohmori (2021) \cite{ohmori2021predictive} (KNN), and (4) the predict and calibrate approach proposed by Sun et al. (2023) \cite{sun2023predict} (Predict and Calibrate). For the predict-then-optimize approach, half of the training data were used to build the core prediction model and the other half were used to predict the residuals.

\medbreak
\noindent \textbf{Results}
\medbreak

\noindent We first report the average radii of uncertainty sets using the classical ellipsoidal, the MSE loss, and the MSEV loss for varying levels of probability guarantees in Table \ref{radius_portfolio}. Our methods perform much better than the classical ellipsoidal uncertainty set, achieving strong probabilistic guarantees with an average radius at times an order of magnitude smaller. As the heteroscedastic noise level $\dsp \Bar{\varepsilon}$ increases, our MSEV loss uncertainty set is able to capture the varying noise level and is therefore able to outrank our MSE loss uncertainty set. The increasing availability of data, from $\dsp N = 1000$ to $\dsp N = 5000$ benefits our methods greatly, particularly that of the MSEV loss. Intuitively, this may be because the MSEV loss relies on accurately predicting both the point prediction and the variance around that prediction.

\begin{table}[H]
\centering
\caption{Average radius of uncertainty sets for varying noise levels $\dsp \Bar{\varepsilon}$ for a given threshold probability of violation $\dsp \alpha$ for the portfolio optimization problem.} \label{radius_portfolio}
\resizebox{\textwidth}{!}{
\begin{tabular}{|| c || c | c | c || c | c | c || c | c | c ||} 
 \hline
$\Bar{\varepsilon}$ & \multicolumn{3}{c||}{$0$} & \multicolumn{3}{c|}{$0.1$}  & \multicolumn{3}{c|}{$1$} \\
 \hline
 $\dsp \alpha$ & 0.1 & 0.05 & 0.01 & 0.1 & 0.05 & 0.01 & 0.1 & 0.05 & 0.01 \\
 \hline\hline
Classical Ellipsoidal & $2.106$ & $2.402$ & $2.978$ & $2.097$ & $2.392$ & $2.965$ & $2.064$ & $2.354$ & $2.918$ \\
\hline
MSE Loss & $\bm{0.021}$ & $\bm{0.023}$ & $\bm{0.031}$ & $\bm{0.349}$ & $\bm{0.407}$ & $\bm{0.570}$ & $2.511$ & $2.939$ & $4.034$ \\
\hline
MSEV Loss & $0.342$ & $0.438$ & $0.740$ & $2.564$ & $3.324$ & $5.727$ & $\bm{1.767}$ & $\bm{2.003}$ & $\bm{2.577}$ \\
\hline
\end{tabular}}
\subcaption{$N = 1000$}
\resizebox{\textwidth}{!}{
\begin{tabular}{|| c || c | c | c || c | c | c || c | c | c ||} 
 \hline
$\Bar{\varepsilon}$ & \multicolumn{3}{c||}{$0$} & \multicolumn{3}{c|}{$0.1$}  & \multicolumn{3}{c|}{$1$} \\
 \hline
 $\dsp \alpha$ & 0.1 & 0.05 & 0.01 & 0.1 & 0.05 & 0.01 & 0.1 & 0.05 & 0.01 \\
 \hline\hline
Classical Ellipsoidal & $2.138$ & $2.438$ & $3.023$ & $2.140$ & $2.440$ & $3.026$ & $2.181$ & $2.488$ & $3.084$ \\
\hline
MSE Loss & $\bm{0.014}$ & $\bm{0.017}$ & $\bm{0.024}$ & $0.287$ & $0.329$ & $0.431$ & $2.448$ & $2.808$ & $3.677$ \\
\hline
MSEV Loss & $0.032$ & $0.041$ & $0.070$ & $\bm{0.224}$ & $\bm{0.250}$ & $\bm{0.307}$ & $\bm{1.857}$ & $\bm{2.080}$ & $\bm{2.594}$ \\
\hline
\end{tabular}}
\subcaption{$N = 5000$}
\resizebox{\textwidth}{!}{
\begin{tabular}{|| c || c | c | c || c | c | c || c | c | c ||} 
 \hline
$\Bar{\varepsilon}$ & \multicolumn{3}{c||}{$0$} & \multicolumn{3}{c|}{$0.1$}  & \multicolumn{3}{c|}{$1$} \\
 \hline
 $\dsp \alpha$ & 0.1 & 0.05 & 0.01 & 0.1 & 0.05 & 0.01 & 0.1 & 0.05 & 0.01 \\
 \hline\hline
Classical Ellipsoidal & $2.166$ & $2.470$ & $3.063$ & $2.162$ & $2.466$ & $3.057$ & $2.090$ & $2.384$ & $2.956$ \\
\hline
MSE Loss & $\bm{0.001}$ & $\bm{0.001}$ & $\bm{0.001}$ & $0.263$ & $0.307$ & $0.412$ & $2.254$ & $2.624$ & $3.516$ \\
\hline
MSEV Loss & $0.001$ & $0.001$ & $0.002$ & $\bm{0.231}$ & $\bm{0.260}$ & $\bm{0.324}$ & $\bm{1.805}$ & $\bm{2.019}$ & $\bm{2.497}$ \\
\hline
\end{tabular}}
\subcaption{$N = 10000$}
\end{table}

We report the trade-off between objective value as a function of the guarantee levels in Figure \ref{portfolio_objective_figures}, for varying levels of noise and number of data points. This reveals a consistent improvement in guarantees from our MSE and MSEV loss over the Classical Ellipsoidal. Notice that as the noise level $\dsp \Bar{\varepsilon}$ increases, these improvements diminish, since the added predictive power of machine learning is hampered by the underlying variance of the data. However, this is remedied by the increase in the number of data points $\dsp N$, unlocking more accurate predictions. Furthermore, the predicted variance of the MSEV loss can further enhance these probabilistic guarantees. Notice that Predict and Calibrate achieves low objectives, but does not offer any meaningful guarantees due to an insufficient number of data points. For completeness, we provide the predicted objective of the SPO approach, although its aim is not to predict the objective cost function accurately, but rather minimize the resulting regret. 

\begin{figure}[H]
\caption{Objective value for the portfolio optimization problem for various $\dsp N$ and $\dsp \Bar{\varepsilon}$.} \label{portfolio_objective_figures}
\begin{tabular}{ccc}
  \includegraphics[trim={0pt 0pt 150pt 0pt}, clip, width=0.29\textwidth] {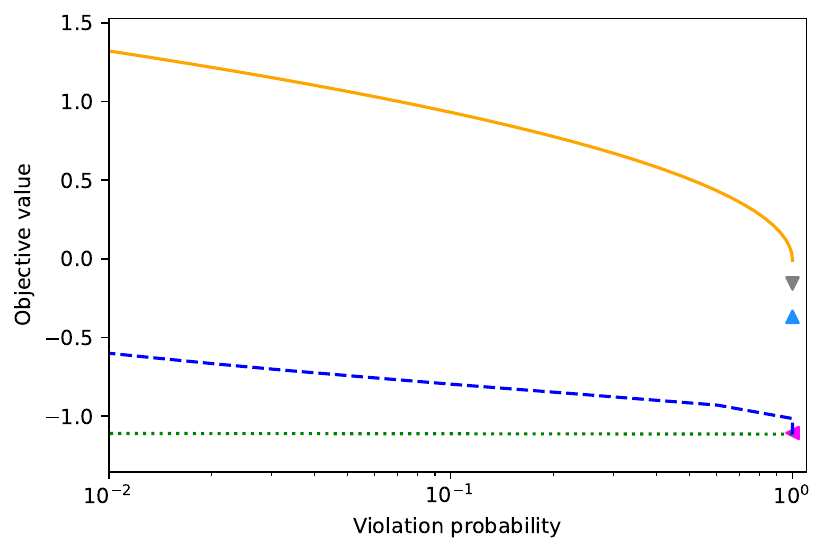} &   \includegraphics[trim={0pt 0pt 150pt 0pt}, clip, width=0.29\textwidth] {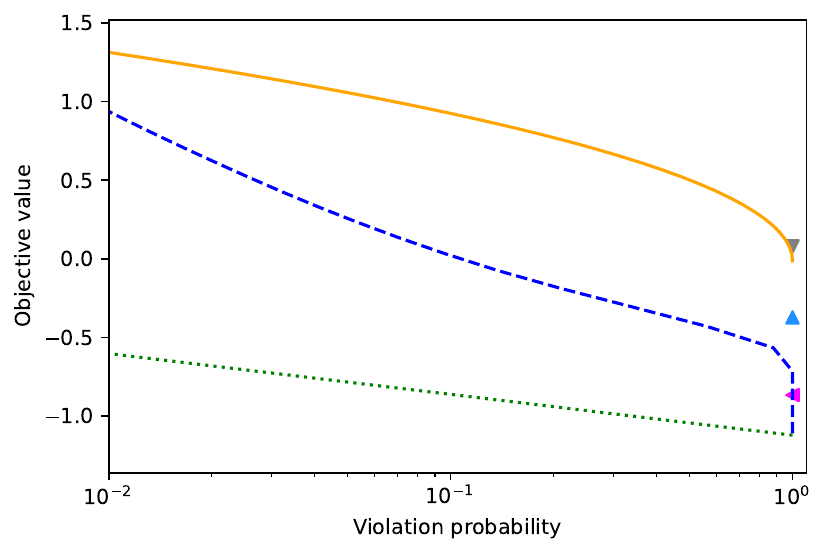} &   \includegraphics[trim={0pt 0pt 0pt 0pt}, clip, width=0.4\textwidth] {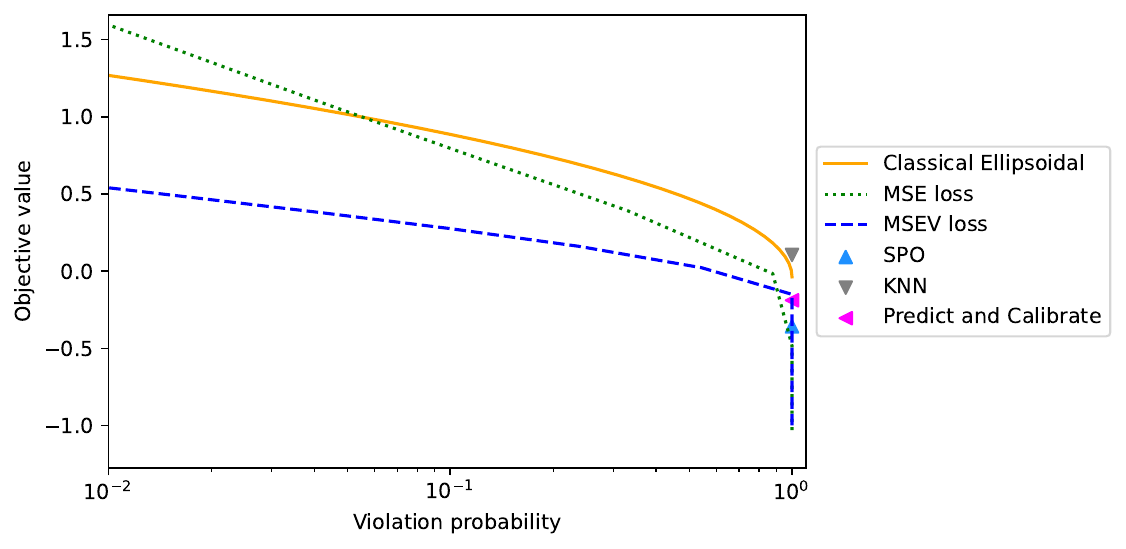}  \\
 \multicolumn{3}{c}{$\dsp N = 1000$} \\[6pt]
 \includegraphics[trim={0pt 0pt 150pt 0pt}, clip, width=0.29\textwidth] {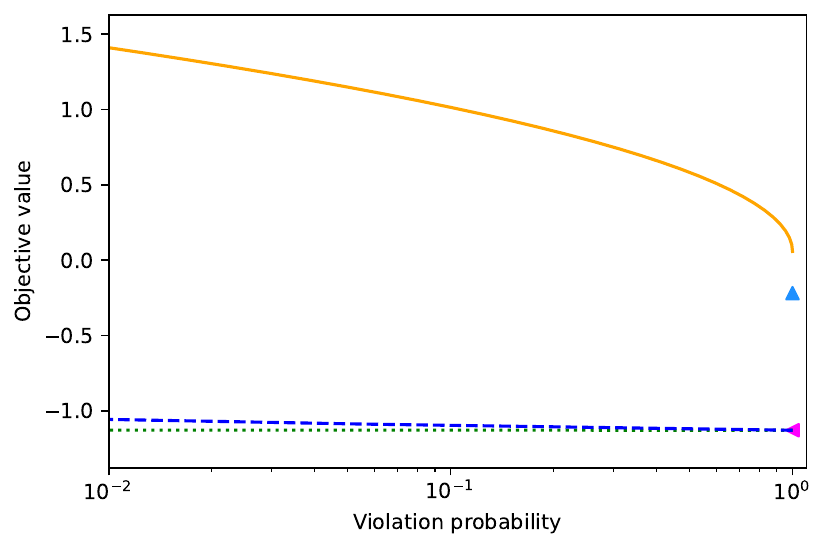} &   \includegraphics[trim={0pt 0pt 150pt 0pt}, clip, width=0.29\textwidth] {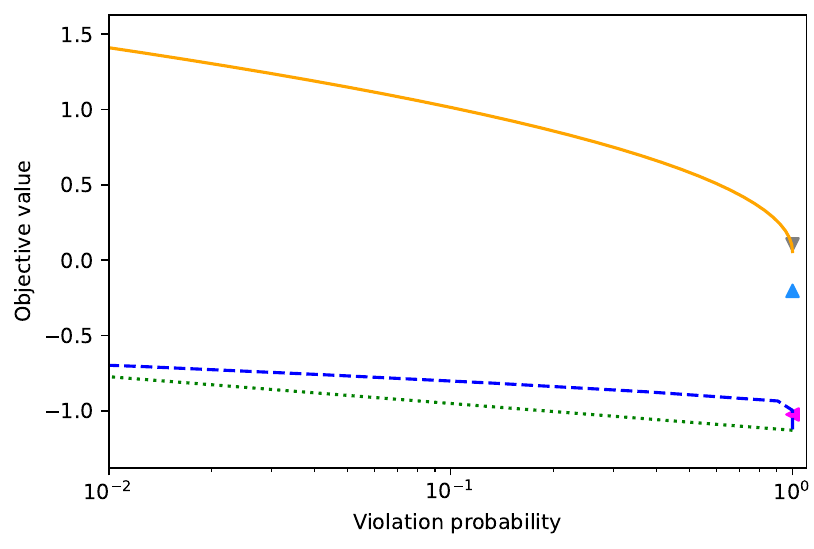} &   \includegraphics[trim={0pt 0pt 0pt 0pt}, clip, width=0.4\textwidth] {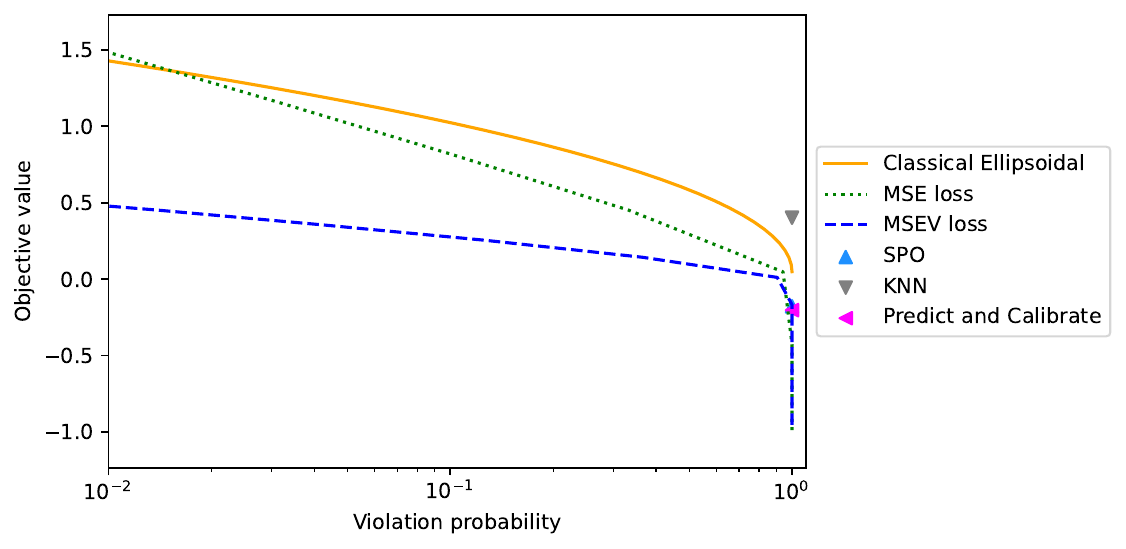} \\
 \multicolumn{3}{c}{$\dsp N = 5000$} \\[6pt]
\includegraphics[trim={0pt 0pt 150pt 0pt}, clip, width=0.29\textwidth] {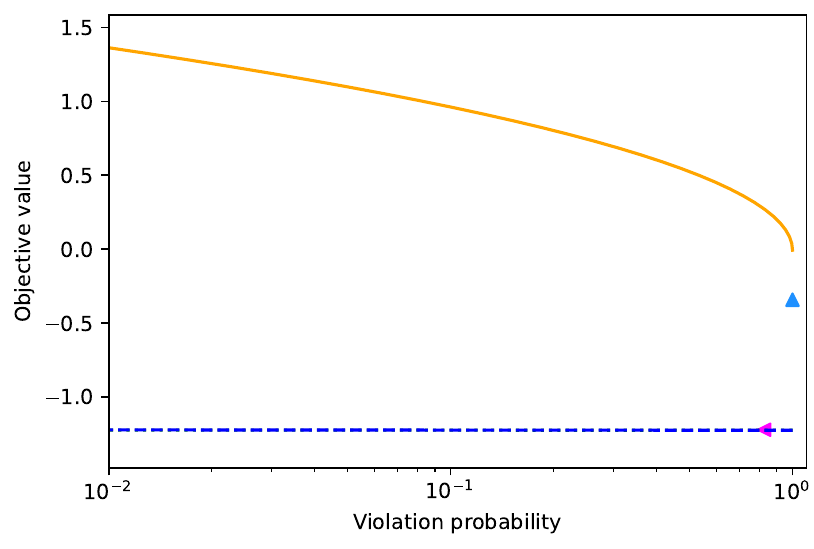} &   \includegraphics[trim={0pt 0pt 150pt 0pt}, clip, width=0.29\textwidth] {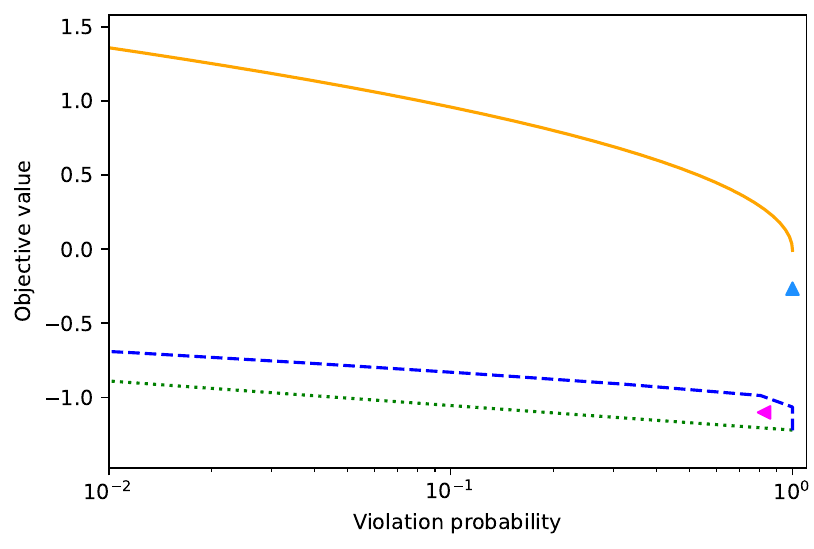} &   \includegraphics[trim={0pt 0pt 0pt 0pt}, clip, width=0.4\textwidth] {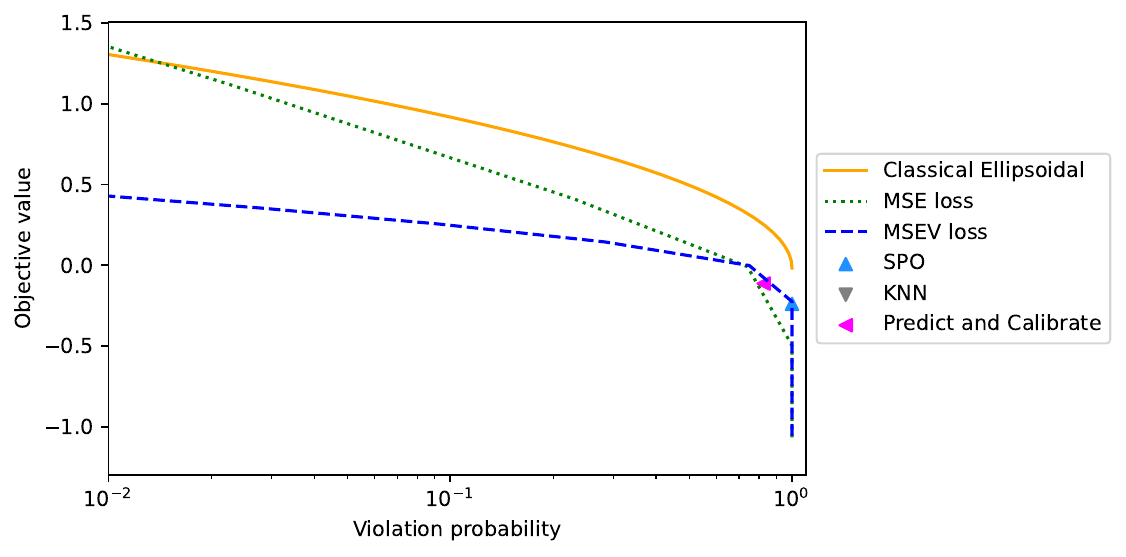} \\
 \multicolumn{3}{c}{$\dsp N = 10000$} \\[6pt]
$\dsp \Bar{\varepsilon} = 0$ & $\dsp \Bar{\varepsilon} = 0.1$ & $\dsp \Bar{\varepsilon} = 1$ \\[6pt]
\end{tabular}
\end{figure}

We report the trade-off between the regret as a function of the guarantee levels in Figure \ref{portfolio_regret_figures}. Our machine learning-based methods vastly reduce the regret compared to the classical ellipsoidal approach since the optimization model has access to estimates of the true values of the returns. Similarly to the objective value trade-offs in Figure \ref{portfolio_objective_figures}, the noise level $\dsp \Bar{\varepsilon}$ increasing diminishes the improvements of our methods over the classical ellipsoidal; however, the increase in the number of data points $\dsp N$ does not seem to play a major factor in diminishing the regret. The SPO and KNN approach are able to diminish the regret over the classical ellipsoidal approach, but not as well as our approaches or the Predict and Calibrate approach.

\begin{figure}[H]
\caption{Regret for the portfolio optimization problem for various $\dsp N$ and $\dsp \Bar{\varepsilon}$.} \label{portfolio_regret_figures}
\begin{tabular}{ccc}
  \includegraphics[trim={0pt 0pt 150pt 0pt}, clip, width=0.29\textwidth] {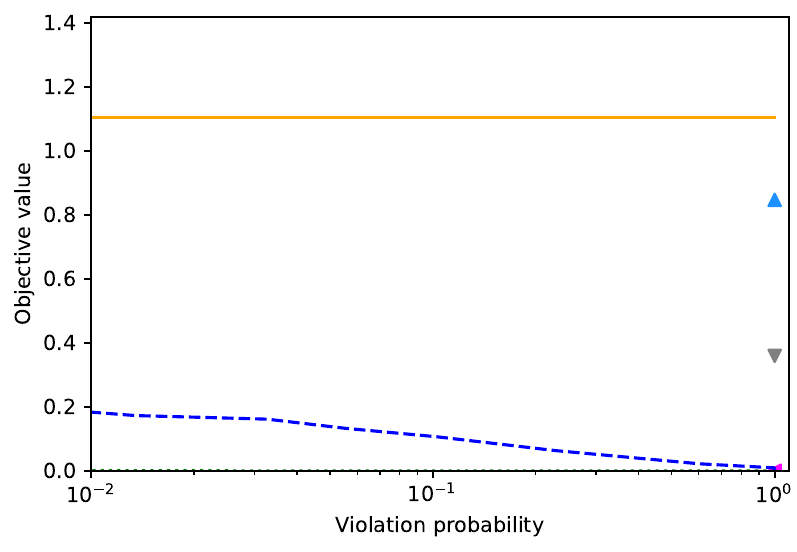} &   \includegraphics[trim={0pt 0pt 150pt 0pt}, clip, width=0.29\textwidth] {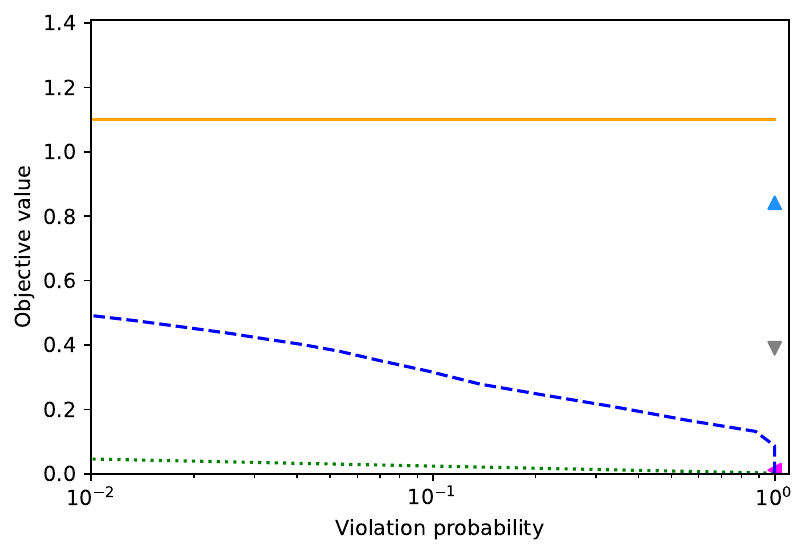} &   \includegraphics[trim={0pt 0pt 0pt 0pt}, clip, width=0.4\textwidth] {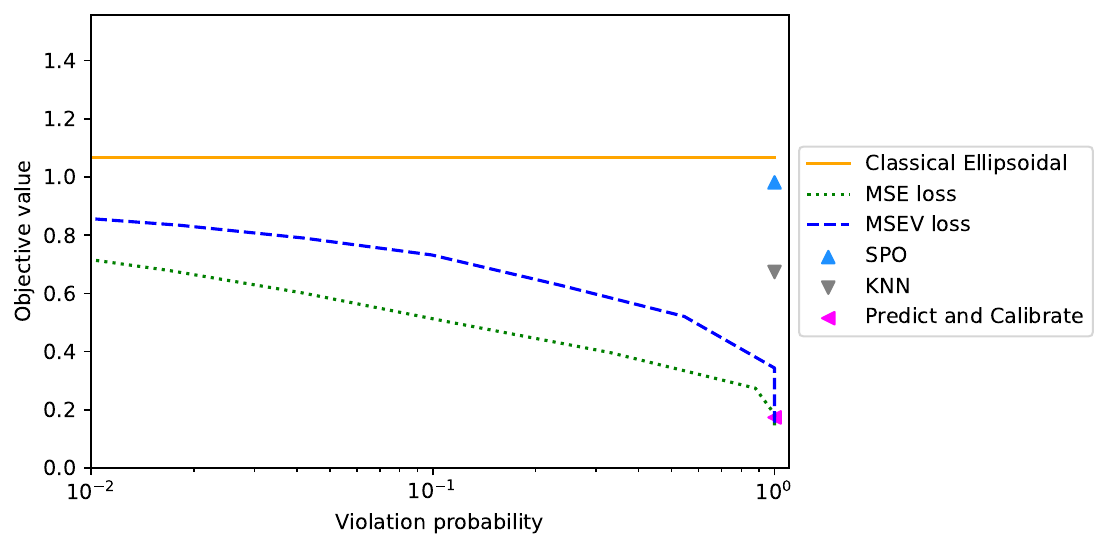}  \\
 \multicolumn{3}{c}{$\dsp N = 1000$} \\[6pt]
 \includegraphics[trim={0pt 0pt 150pt 0pt}, clip, width=0.29\textwidth] {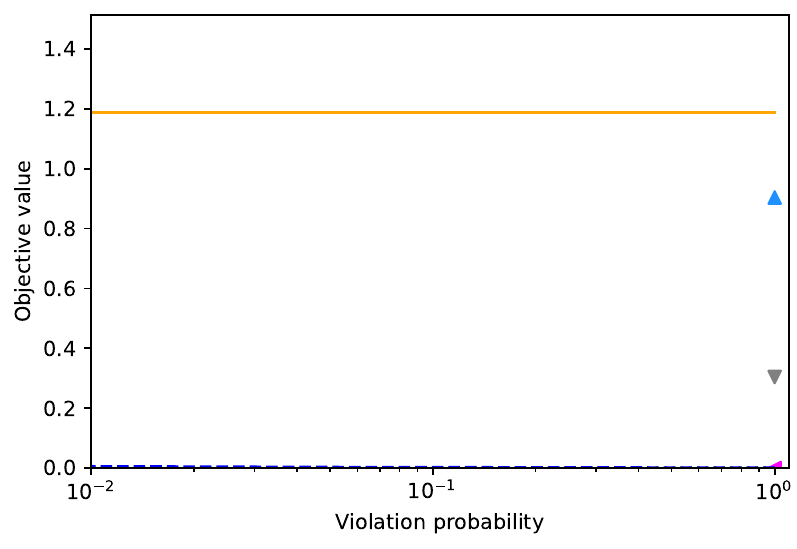} &   \includegraphics[trim={0pt 0pt 150pt 0pt}, clip, width=0.29\textwidth] {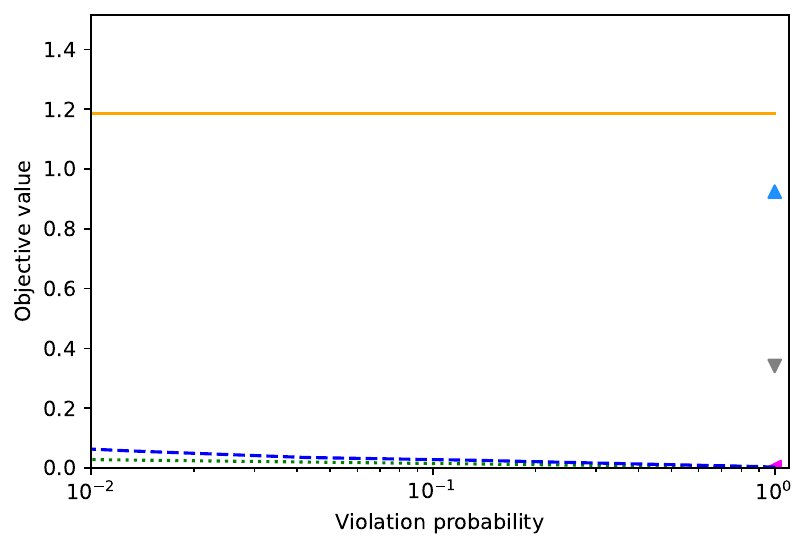} &   \includegraphics[trim={0pt 0pt 0pt 0pt}, clip, width=0.4\textwidth] {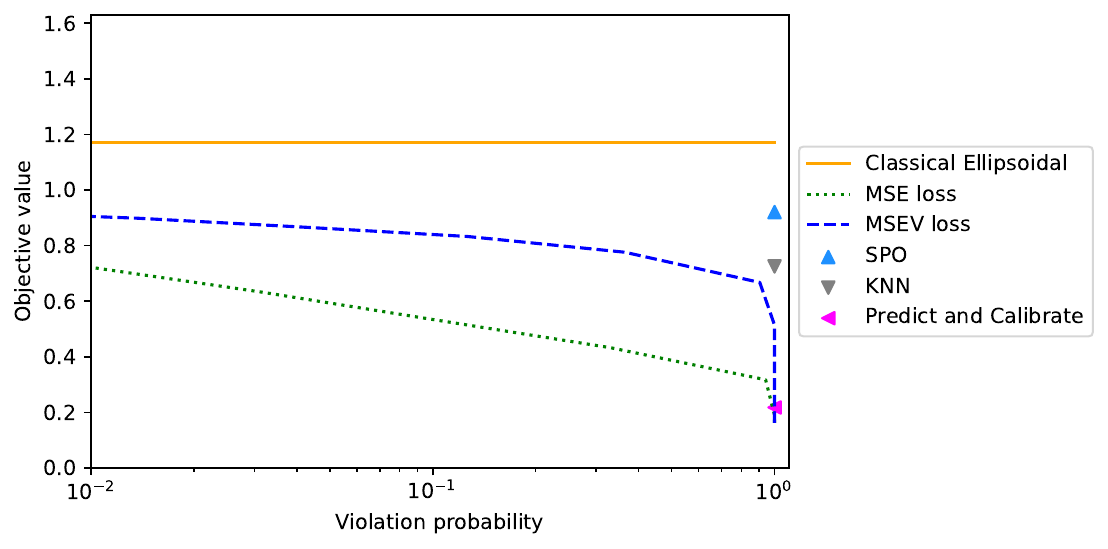} \\
 \multicolumn{3}{c}{$\dsp N = 5000$} \\[6pt]
\includegraphics[trim={0pt 0pt 150pt 0pt}, clip, width=0.29\textwidth] {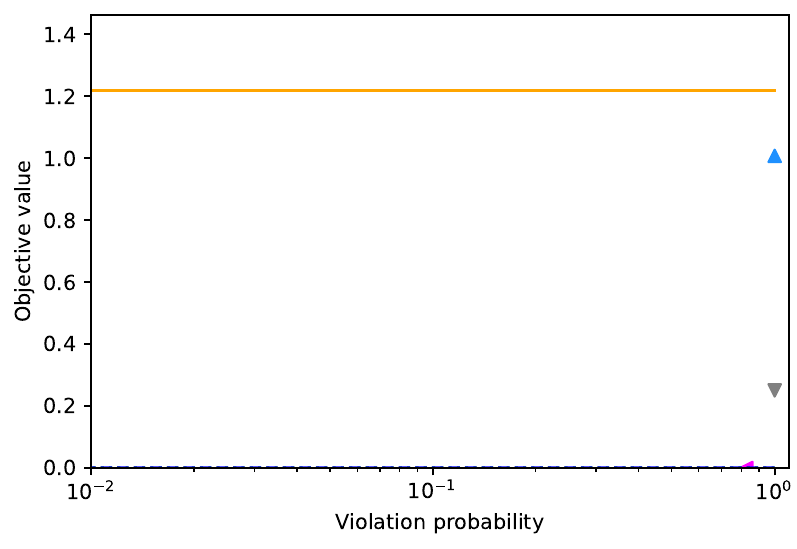} &   \includegraphics[trim={0pt 0pt 150pt 0pt}, clip, width=0.29\textwidth] {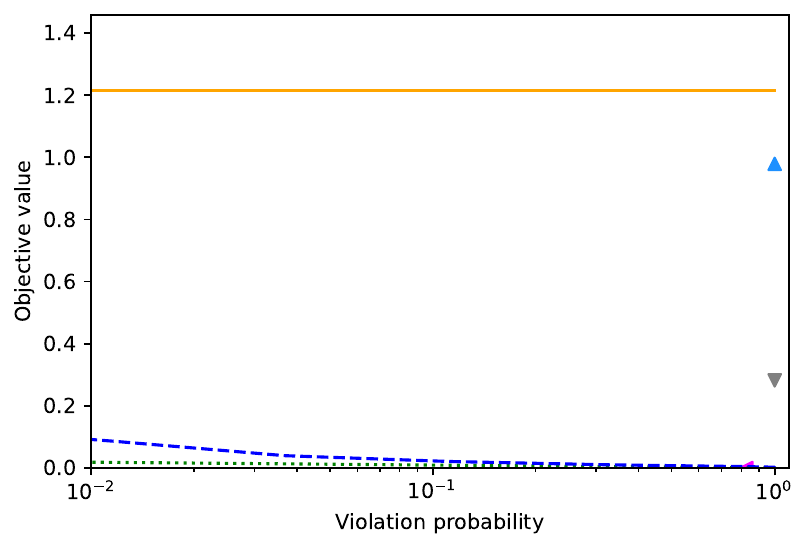} &   \includegraphics[trim={0pt 0pt 0pt 0pt}, clip, width=0.4\textwidth] {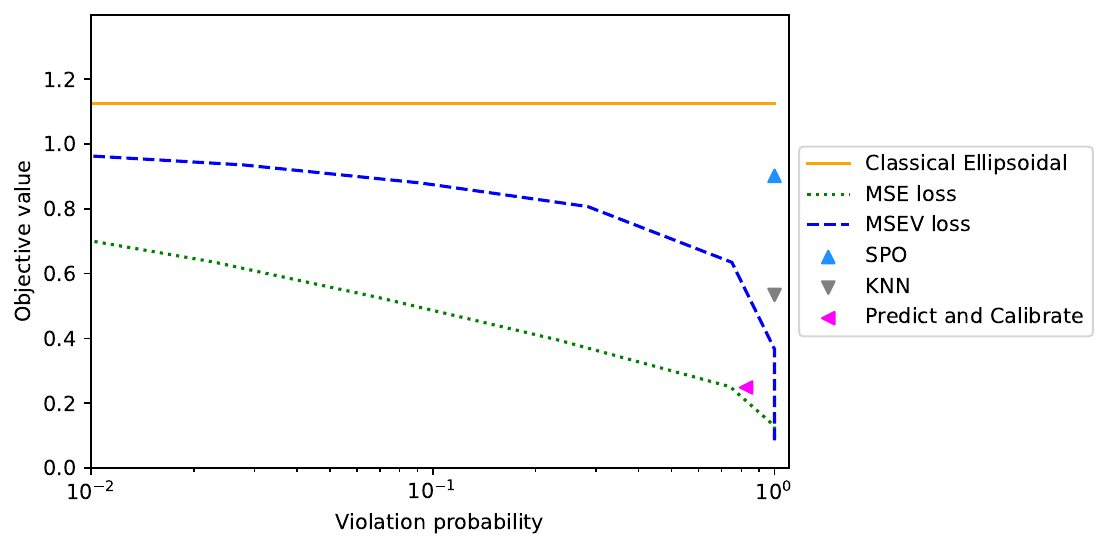} \\
 \multicolumn{3}{c}{$\dsp N = 10000$} \\[6pt]
$\dsp \Bar{\varepsilon} = 0$ & $\dsp \Bar{\varepsilon} = 0.1$ & $\dsp \Bar{\varepsilon} = 1$ \\[6pt]
\end{tabular}
\end{figure}

\subsection{Shortest Path}

We consider a shortest path problem on a $5 \times 5$ grid network, as studied by Elmachtoub and Grigas (2022) \cite{elmachtoub2022smart}. The objective is to determine the path with lowest cost originating from the northwest corner and ending up in the southeast corner, while only moving east and south along edges. Each 40 edges has a unique cost generated from our data.

The uncertain parameters' data was generated as follows: first, by generating an independent, standard normal vector of covariates $\dsp \bm{X} \in \bbR^{N \times p}$ as well as a vector $\dsp \bm{B} \in \mathbb{R}^{p}$ where each entry follows a Bernoulli distribution that equals $\dsp 1$ with a probability of $\dsp 0.5$; next, by sampling a multiplicative noise uniformly $\dsp \bm{\varepsilon} \in \bbR^{N}$ in $\left[ 1 -
\Bar{\varepsilon}, 1 +
\Bar{\varepsilon} \right]$, where $\dsp \Bar{\varepsilon} \geq 0$; finally, by building the vector of uncertain parameters $\dsp \bm{y} = \left[ \left( \frac{1}{\sqrt{p}} \left( \bm{X} \bm{B} \right) + 3 \right)^{\text{deg}} + 1 \right] \cdot \bm{\varepsilon}$ where $\dsp \text{deg}$ is a fixed positive integer, which we then standardize.

For our experiments, we set the number of covariates $\dsp p = 10$ and the number of training data points $\dsp N = 5000$. We report results on an independent test set with  $\dsp 100$ samples. We experimented with noise levels $\dsp \Bar{\varepsilon} \in \left\{ 0.1, \, 0.3, \, 0.5 \right\}$ and degree parameters $\dsp \text{deg} \in \left\{ 1, \, 2, \, 4 \right\}$. We compare the same algorithms as in the experiment in Section \ref{portfolio_optimization}.

The results are similar, and we relegate them to the Appendix.

\section{Conclusion}

This paper proposes a robust optimization approach to protect against uncertainties inherent in machine learning model outputs, within an optimization formulation. We develop an uncertainty set based on the loss function of the machine learning model, drawing from established literature on loss functions. Our method provides general guarantees for our uncertainty sets, as well as strong probabilistic guarantees for the most common regression loss function. Empirical evidence demonstrates that our approach not only enhances the guarantees found in existing literature but also provides low objective value and regret. Future work might focus on developing stronger probabilistic guarantees to a wider class of loss functions for classic optimization problems.

\bibliographystyle{plain} % We choose the "plain" reference style
\bibliography{bibliography} % Entries are in the refs.bib file

\begin{thebibliography}{10}

\bibitem{ben2013robust}
Aharon Ben-Tal, Dick Den~Hertog, Anja De~Waegenaere, Bertrand Melenberg, and Gijs Rennen.
\newblock Robust solutions of optimization problems affected by uncertain probabilities.
\newblock {\em Management Science}, 59(2):341--357, 2013.

\bibitem{ben2015deriving}
Aharon Ben-Tal, Dick Den~Hertog, and Jean-Philippe Vial.
\newblock Deriving robust counterparts of nonlinear uncertain inequalities.
\newblock {\em Mathematical programming}, 149(1):265--299, 2015.

\bibitem{ben2000robust}
Aharon Ben-Tal and Arkadi Nemirovski.
\newblock Robust solutions of linear programming problems contaminated with uncertain data.
\newblock {\em Mathematical programming}, 88:411--424, 2000.

\bibitem{bertsimas2021probabilistic}
Dimitris Bertsimas, Dick Den~Hertog, and Jean Pauphilet.
\newblock Probabilistic guarantees in robust optimization.
\newblock {\em SIAM Journal on Optimization}, 31(4):2893--2920, 2021.

\bibitem{bertsimas2017optimal}
Dimitris Bertsimas and Jack Dunn.
\newblock Optimal classification trees.
\newblock {\em Machine Learning}, 106:1039--1082, 2017.

\bibitem{bertsimas2019robust}
Dimitris Bertsimas, Jack Dunn, Colin Pawlowski, and Ying~Daisy Zhuo.
\newblock Robust classification.
\newblock {\em INFORMS Journal on Optimization}, 1(1):2--34, 2019.

\bibitem{bertsimas2020predictive}
Dimitris Bertsimas and Nathan Kallus.
\newblock From predictive to prescriptive analytics.
\newblock {\em Management Science}, 66(3):1025--1044, 2020.

\bibitem{bishop1994mixture}
Christopher~M Bishop.
\newblock Mixture density networks.
\newblock 1994.

\bibitem{cortes1995support}
Corinna Cortes and Vladimir Vapnik.
\newblock Support-vector networks.
\newblock {\em Machine learning}, 20:273--297, 1995.

\bibitem{delage2010distributionally}
Erick Delage and Yinyu Ye.
\newblock Distributionally robust optimization under moment uncertainty with application to data-driven problems.
\newblock {\em Operations research}, 58(3):595--612, 2010.

\bibitem{elmachtoub2022smart}
Adam~N Elmachtoub and Paul Grigas.
\newblock Smart “predict, then optimize”.
\newblock {\em Management Science}, 68(1):9--26, 2022.

\bibitem{he2014practical}
Xinran He, Junfeng Pan, Ou~Jin, Tianbing Xu, Bo~Liu, Tao Xu, Yanxin Shi, Antoine Atallah, Ralf Herbrich, Stuart Bowers, et~al.
\newblock Practical lessons from predicting clicks on ads at facebook.
\newblock In {\em Proceedings of the eighth international workshop on data mining for online advertising}, pages 1--9, 2014.

\bibitem{kendall2018multi}
Alex Kendall, Yarin Gal, and Roberto Cipolla.
\newblock Multi-task learning using uncertainty to weigh losses for scene geometry and semantics.
\newblock In {\em Proceedings of the IEEE conference on computer vision and pattern recognition}, pages 7482--7491, 2018.

\bibitem{mutapcic2009cutting}
Almir Mutapcic and Stephen Boyd.
\newblock Cutting-set methods for robust convex optimization with pessimizing oracles.
\newblock {\em Optimization Methods \& Software}, 24(3):381--406, 2009.

\bibitem{nix1994estimating}
David~A Nix and Andreas~S Weigend.
\newblock Estimating the mean and variance of the target probability distribution.
\newblock In {\em Proceedings of 1994 ieee international conference on neural networks (ICNN'94)}, volume~1, pages 55--60. IEEE, 1994.

\bibitem{ohmori2021predictive}
Shunichi Ohmori.
\newblock A predictive prescription using minimum volume k-nearest neighbor enclosing ellipsoid and robust optimization.
\newblock {\em Mathematics}, 9(2):119, 2021.

\bibitem{pardo2018statistical}
Leandro Pardo.
\newblock {\em Statistical inference based on divergence measures}.
\newblock Chapman and Hall/CRC, 2018.

\bibitem{ruder2017overview}
Sebastian Ruder.
\newblock An overview of multi-task learning in deep neural networks.
\newblock {\em arXiv preprint arXiv:1706.05098}, 2017.

\bibitem{shapiro2021lectures}
Alexander Shapiro, Darinka Dentcheva, and Andrzej Ruszczynski.
\newblock {\em Lectures on stochastic programming: modeling and theory}.
\newblock SIAM, 2021.

\bibitem{skafte2019reliable}
Nicki Skafte, Martin J{\o}rgensen, and S{\o}ren Hauberg.
\newblock Reliable training and estimation of variance networks.
\newblock {\em Advances in Neural Information Processing Systems}, 32, 2019.

\bibitem{sun2023predict}
Chunlin Sun, Linyu Liu, and Xiaocheng Li.
\newblock Predict-then-calibrate: A new perspective of robust contextual lp.
\newblock {\em Advances in Neural Information Processing Systems}, 36:17713--17741, 2023.

\bibitem{tulabandhula2014robust}
Theja Tulabandhula and Cynthia Rudin.
\newblock Robust optimization using machine learning for uncertainty sets.
\newblock {\em arXiv preprint arXiv:1407.1097}, 2014.

\end{thebibliography}

\appendix

\section{Additional Results} \label{complete_results}

Similarly to the experiment from Section \ref{portfolio_optimization}, Table \ref{radius_shortest_path} shows that the radii of our machine learning based uncertainty sets shrink as the noise level $\dsp \Bar{\varepsilon}$ shrinks. The radii also shrink as the degree parameter $\dsp \text{deg}$ increases, since a high degree parameter increases the overall variance of the data distribution faster than the conditional variance given the covariates.

\begin{table}[H]
\centering
\caption{Average radius of uncertainty sets for varying noise levels $\dsp \Bar{\varepsilon}$ for a given threshold probability of violation $\dsp \alpha$ for the shortest path problem.} \label{radius_shortest_path}
\resizebox{\textwidth}{!}{
\begin{tabular}{|| c || c | c | c || c | c | c || c | c | c ||} 
 \hline
$\Bar{\varepsilon}$ & \multicolumn{3}{c||}{$0.1$} & \multicolumn{3}{c|}{$0.2$}  & \multicolumn{3}{c|}{$0.3$} \\
 \hline
 $\dsp \alpha$ & 0.1 & 0.05 & 0.01 & 0.1 & 0.05 & 0.01 & 0.1 & 0.05 & 0.01 \\
 \hline\hline
Classical Ellipsoidal & $2.133$ & $2.433$ & $3.017$ & $2.131$ & $2.430$ & $3.013$ & $\bm{2.127}$ & $\bm{2.426}$ & $\bm{3.008}$ \\
\hline
MSE Loss & $\bm{0.934}$ & $\bm{1.002}$ & $\bm{1.160}$ & $\bm{1.840}$ & $\bm{1.977}$ & $\bm{2.295}$ & $2.684$ & $2.887$ & $3.360$ \\
\hline
MSEV Loss & $1.152$ & $1.392$ & $2.095$ & $2.036$ & $2.418$ & $3.579$ & $2.904$ & $3.424$ & $5.010$ \\
\hline
\end{tabular}}
\subcaption{$\text{deg} = 1$}
\resizebox{\textwidth}{!}{
\begin{tabular}{|| c || c | c | c || c | c | c || c | c | c ||} 
 \hline
$\Bar{\varepsilon}$ & \multicolumn{3}{c||}{$0.1$} & \multicolumn{3}{c|}{$0.2$}  & \multicolumn{3}{c|}{$0.3$} \\
 \hline
 $\dsp \alpha$ & 0.1 & 0.05 & 0.01 & 0.1 & 0.05 & 0.01 & 0.1 & 0.05 & 0.01 \\
 \hline\hline
Classical Ellipsoidal & $2.068$ & $2.359$ & $2.925$ & $2.068$ & $2.359$ & $2.925$ & $2.068$ & $2.359$ & $2.925$ \\
\hline
MSE Loss & $2.750$ & $3.276$ & $4.709$ & $3.109$ & $3.701$ & $5.415$ & $5.103$ & $5.945$ & $8.321$ \\
\hline
MSEV Loss & $\bm{0.807}$ & $\bm{0.854}$ & $\bm{0.969}$ & $\bm{1.048}$ & $\bm{1.107}$ & $\bm{1.243}$ & $\bm{1.791}$ & $\bm{1.896}$ & $\bm{2.147}$ \\
\hline
\end{tabular}}
\subcaption{$\text{deg} = 2$}
\resizebox{\textwidth}{!}{
\begin{tabular}{|| c || c | c | c || c | c | c || c | c | c ||} 
 \hline
$\Bar{\varepsilon}$ & \multicolumn{3}{c||}{$0.1$} & \multicolumn{3}{c|}{$0.2$}  & \multicolumn{3}{c|}{$0.3$} \\
 \hline
 $\dsp \alpha$ & 0.1 & 0.05 & 0.01 & 0.1 & 0.05 & 0.01 & 0.1 & 0.05 & 0.01 \\
 \hline\hline
Classical Ellipsoidal & $1.787$ & $2.039$ & $2.528$ & $1.809$ & $2.063$ & $2.558$ & $1.828$ & $2.085$ & $2.585$ \\
\hline
MSE Loss & $6.429$ & $8.043$ & $13.111$ & $6.506$ & $8.183$ & $13.509$ & $9.294$ & $11.435$ & $17.991$ \\
\hline
MSEV Loss & $\bm{0.480}$ & $\bm{0.513}$ & $\bm{0.598}$ & $\bm{0.713}$ & $\bm{0.757}$ & $\bm{0.861}$ & $\bm{1.189}$ & $\bm{1.266}$ & $\bm{1.454}$ \\
\hline
\end{tabular}}
\subcaption{$\text{deg} = 4$}
\end{table}

Even when the radius of the classical ellipsoidal is smaller than that of our machine learning based sets, our methods achieve lower objectives and regret as can be observed in Figure \ref{shortest_path_objective_figures} and Figure \ref{shortest_path_regret_figures}.

\begin{figure}[H]
\caption{Objective value for the shortest path problem for various $\dsp \text{deg}$ and $\dsp \Bar{\varepsilon}$.} \label{shortest_path_objective_figures}
\begin{tabular}{ccc}
  \includegraphics[trim={0pt 0pt 150pt 0pt}, clip, width=0.29\textwidth] {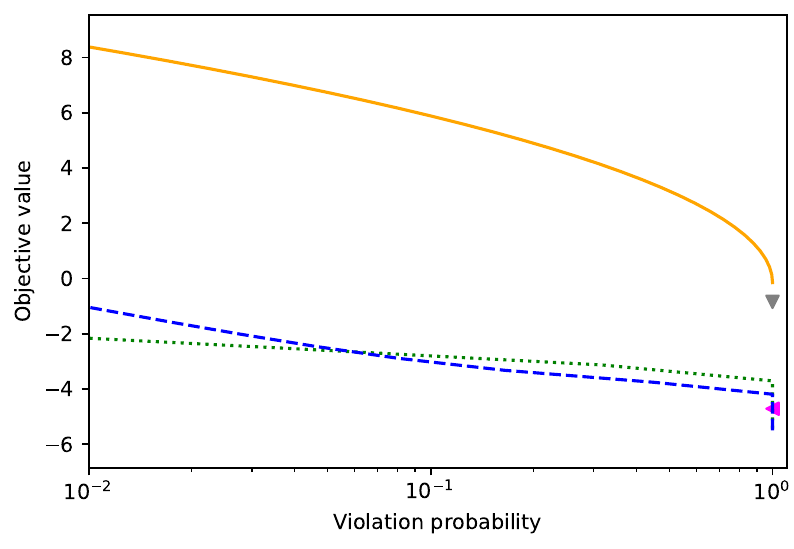} &   \includegraphics[trim={0pt 0pt 150pt 0pt}, clip, width=0.29\textwidth] {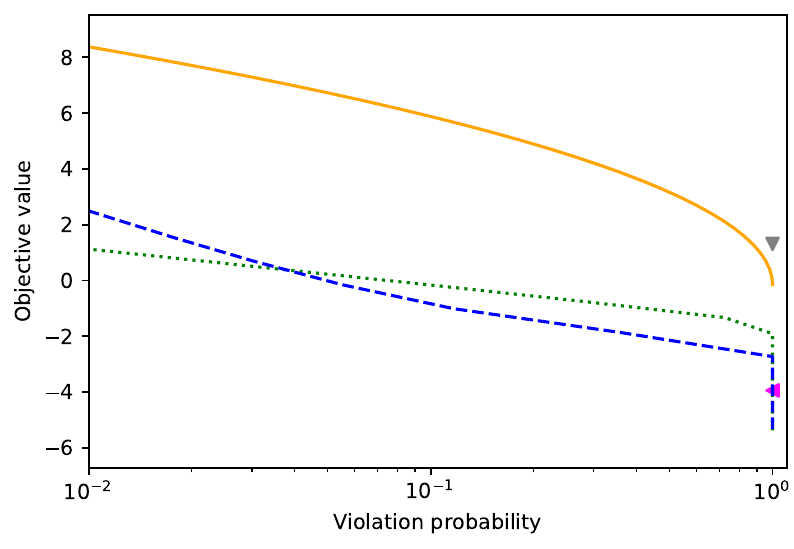} &   \includegraphics[trim={0pt 0pt 0pt 0pt}, clip, width=0.4\textwidth] {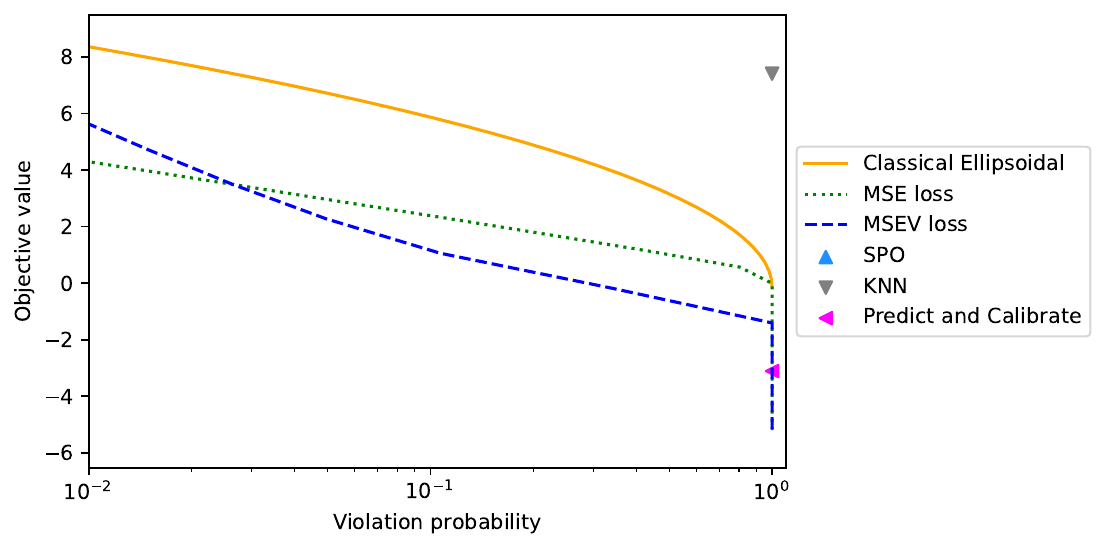}  \\
 \multicolumn{3}{c}{$\text{deg} = 1$} \\[6pt]
 \includegraphics[trim={0pt 0pt 150pt 0pt}, clip, width=0.29\textwidth] {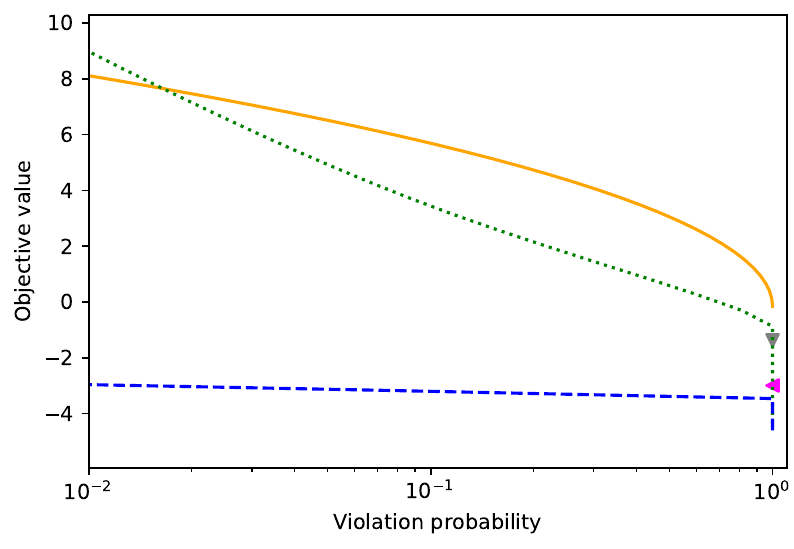} &   \includegraphics[trim={0pt 0pt 150pt 0pt}, clip, width=0.29\textwidth] {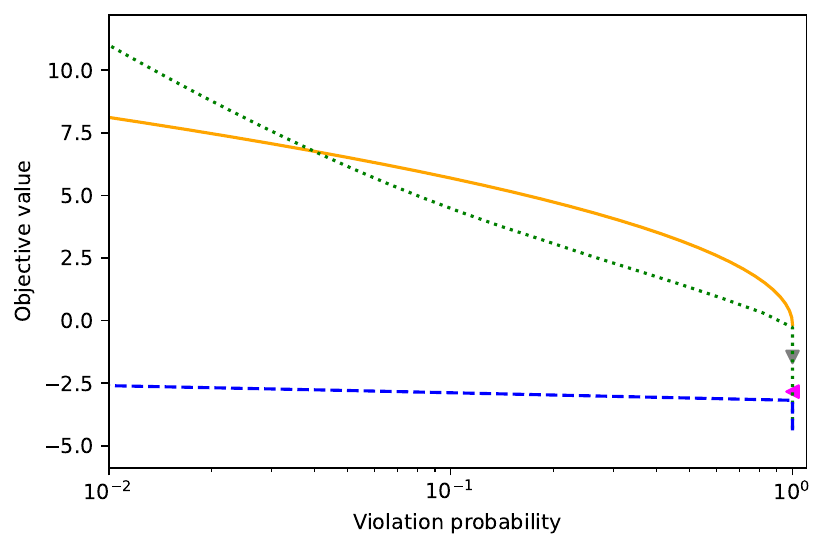} &   \includegraphics[trim={0pt 0pt 0pt 0pt}, clip, width=0.4\textwidth] {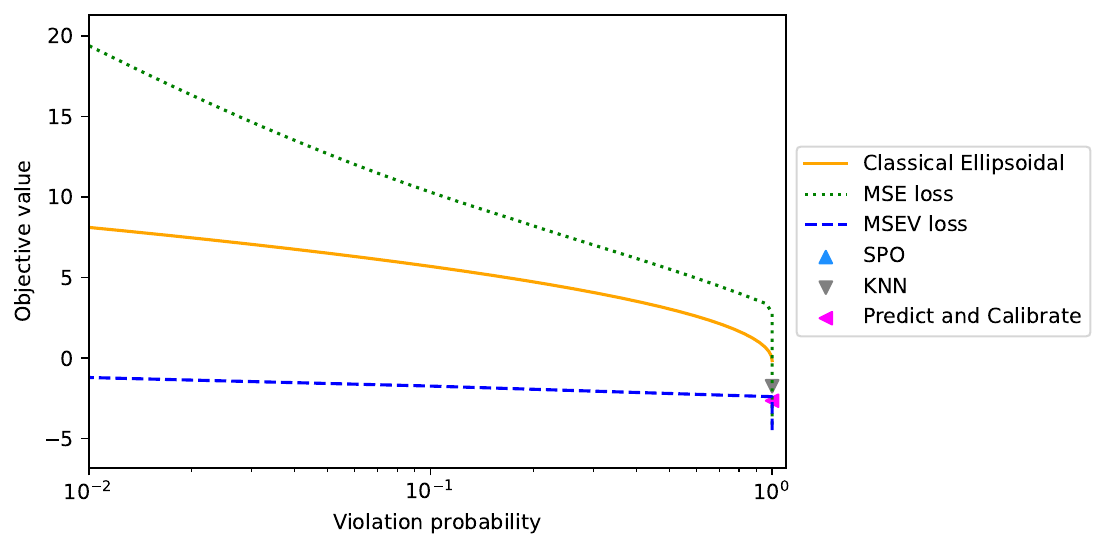} \\
 \multicolumn{3}{c}{$\text{deg} = 2$} \\[6pt]
\includegraphics[trim={0pt 0pt 150pt 0pt}, clip, width=0.29\textwidth] {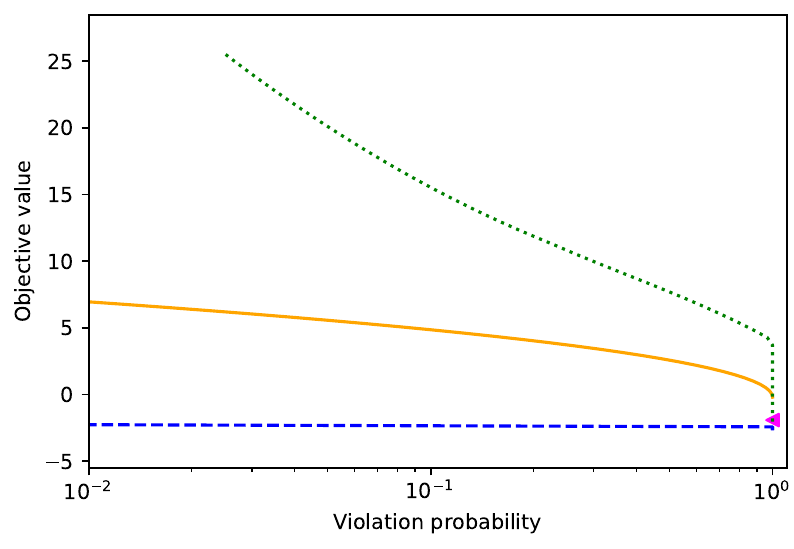} &   \includegraphics[trim={0pt 0pt 150pt 0pt}, clip, width=0.29\textwidth] {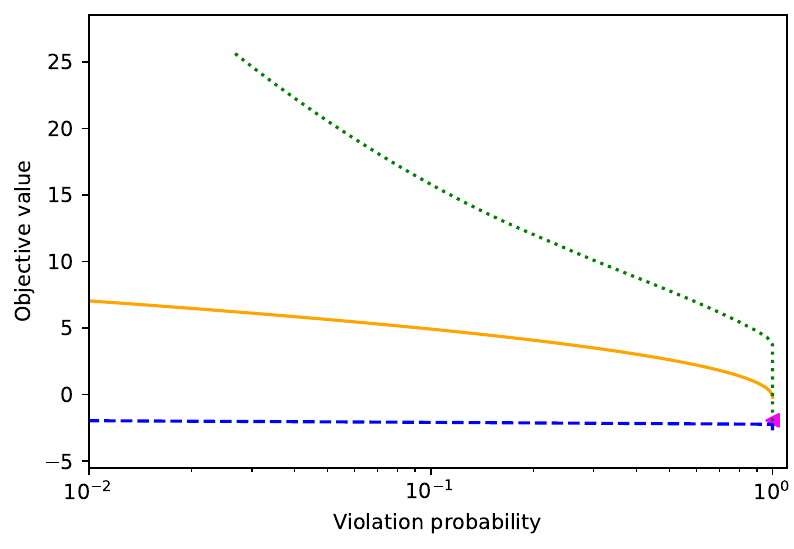} &   \includegraphics[trim={0pt 0pt 0pt 0pt}, clip, width=0.4\textwidth] {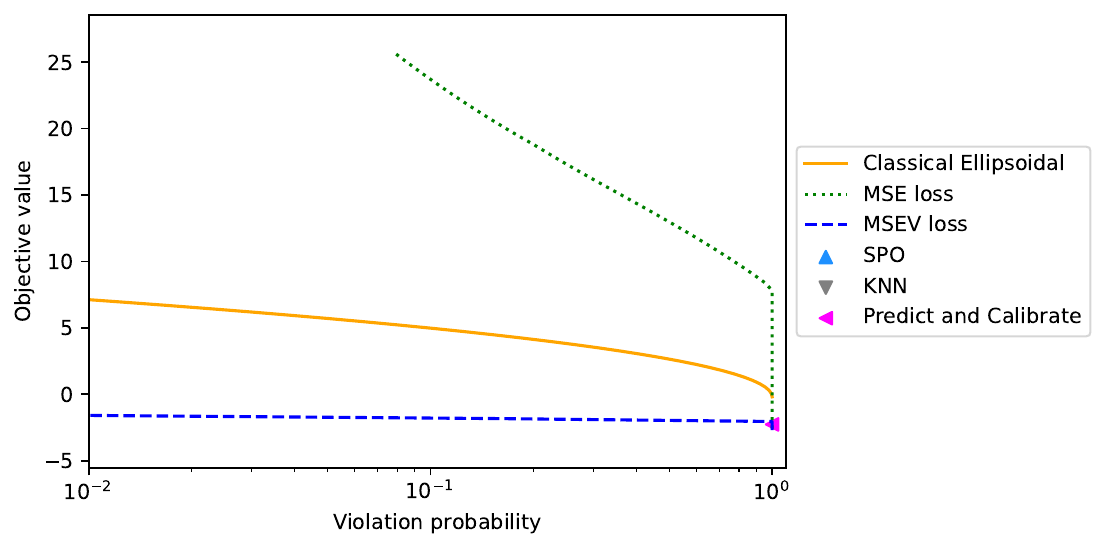} \\
 \multicolumn{3}{c}{$\text{deg} = 4$} \\[6pt]
$\dsp \Bar{\varepsilon} = 0.1$ & $\dsp \Bar{\varepsilon} = 0.3$ & $\dsp \Bar{\varepsilon} = 0.5$ \\[6pt]
\end{tabular}
\end{figure}

\begin{figure}[H]
\caption{Regret for the shortest path problem for various $\dsp \text{deg}$ and $\dsp \Bar{\varepsilon}$.} \label{shortest_path_regret_figures}
\begin{tabular}{ccc}
  \includegraphics[trim={0pt 0pt 150pt 0pt}, clip, width=0.29\textwidth] {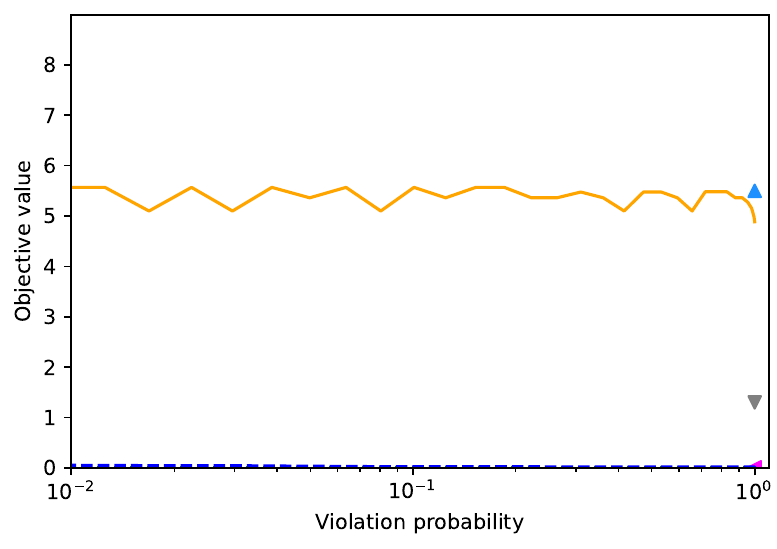} &   \includegraphics[trim={0pt 0pt 150pt 0pt}, clip, width=0.29\textwidth] {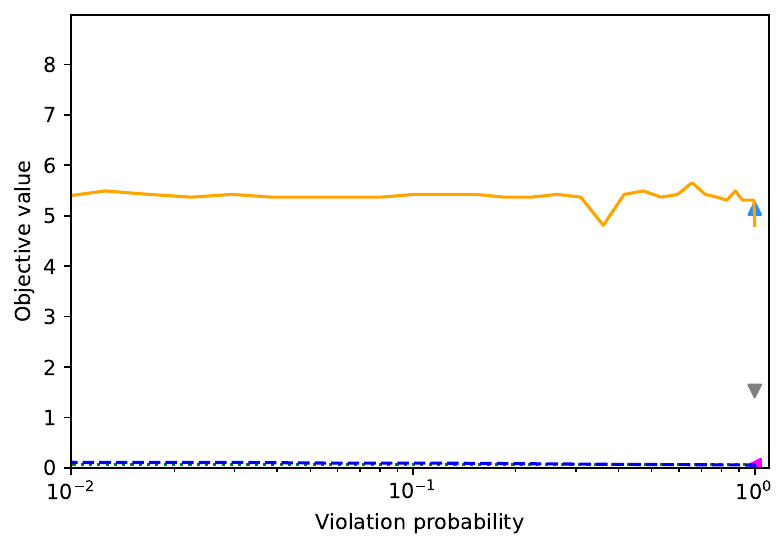} &   \includegraphics[trim={0pt 0pt 0pt 0pt}, clip, width=0.4\textwidth] {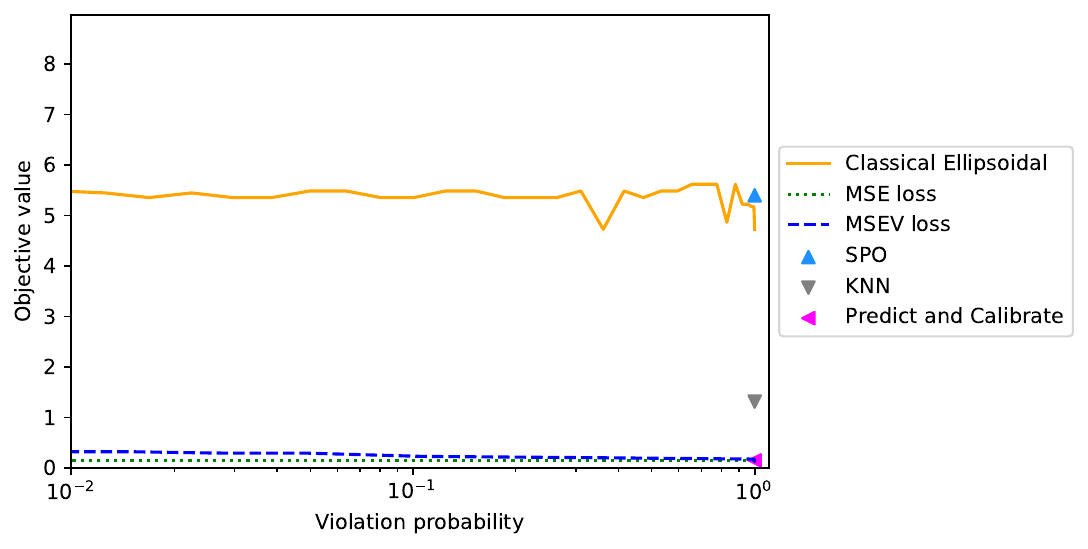}  \\
 \multicolumn{3}{c}{$\text{deg} = 1$} \\[6pt]
 \includegraphics[trim={0pt 0pt 150pt 0pt}, clip, width=0.29\textwidth] {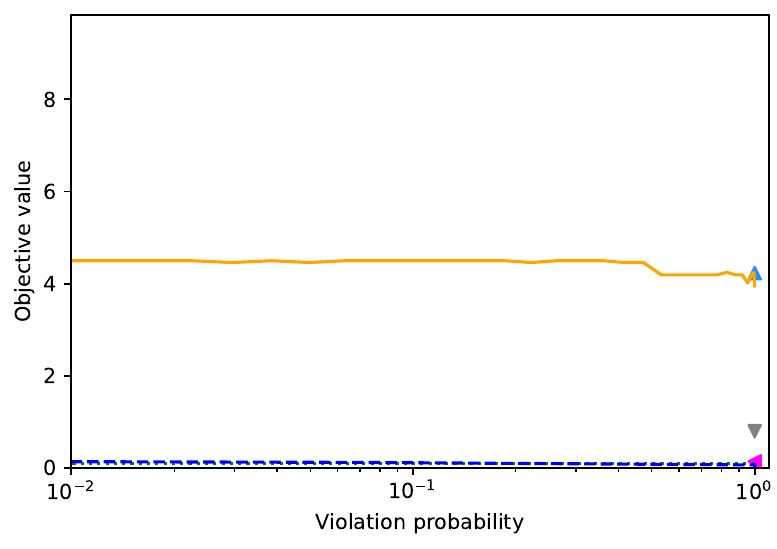} &   \includegraphics[trim={0pt 0pt 150pt 0pt}, clip, width=0.29\textwidth] {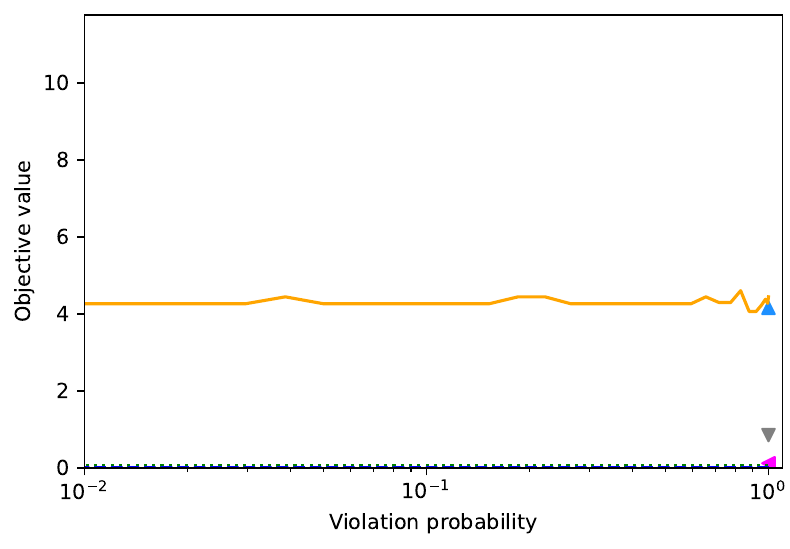} &   \includegraphics[trim={0pt 0pt 0pt 0pt}, clip, width=0.4\textwidth] {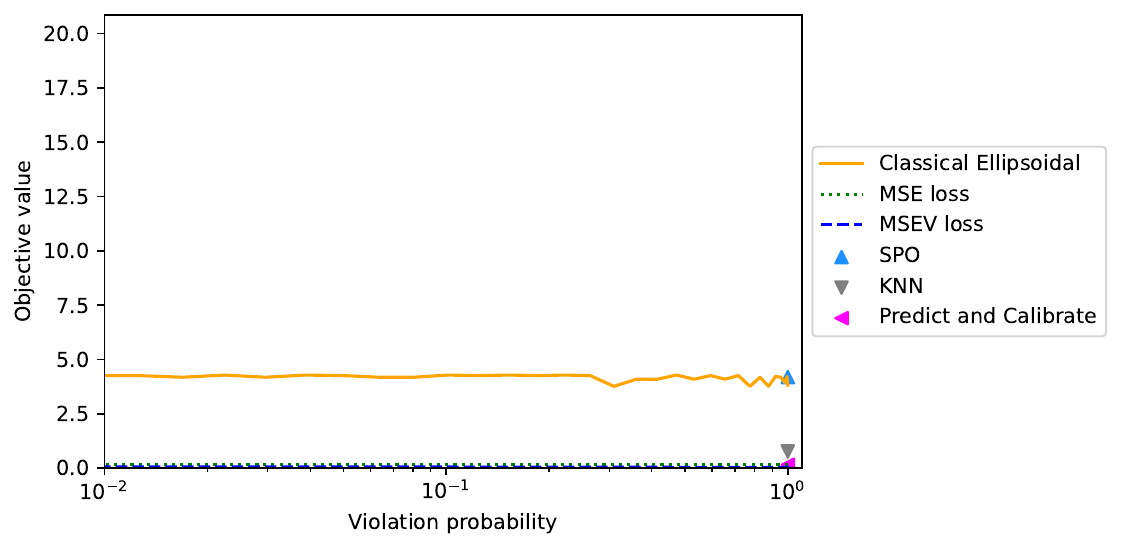} \\
 \multicolumn{3}{c}{$\text{deg} = 2$} \\[6pt]
\includegraphics[trim={0pt 0pt 150pt 0pt}, clip, width=0.29\textwidth] {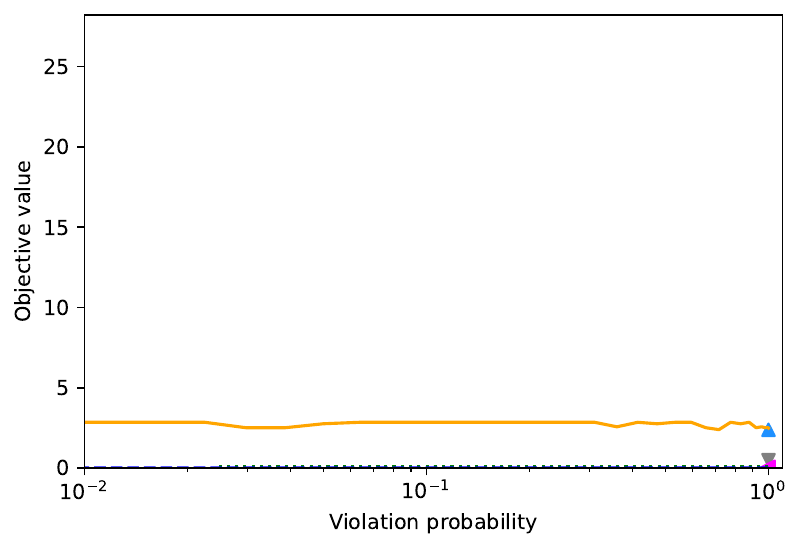} &   \includegraphics[trim={0pt 0pt 150pt 0pt}, clip, width=0.29\textwidth] {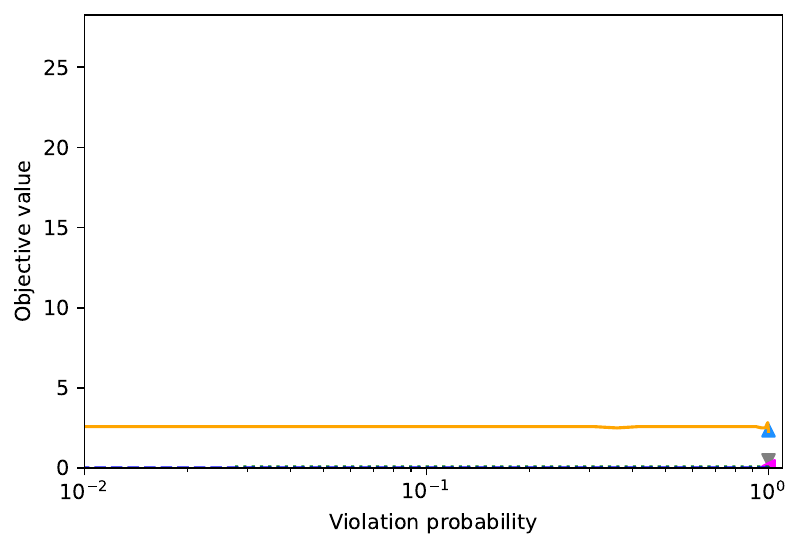} &   \includegraphics[trim={0pt 0pt 0pt 0pt}, clip, width=0.4\textwidth] {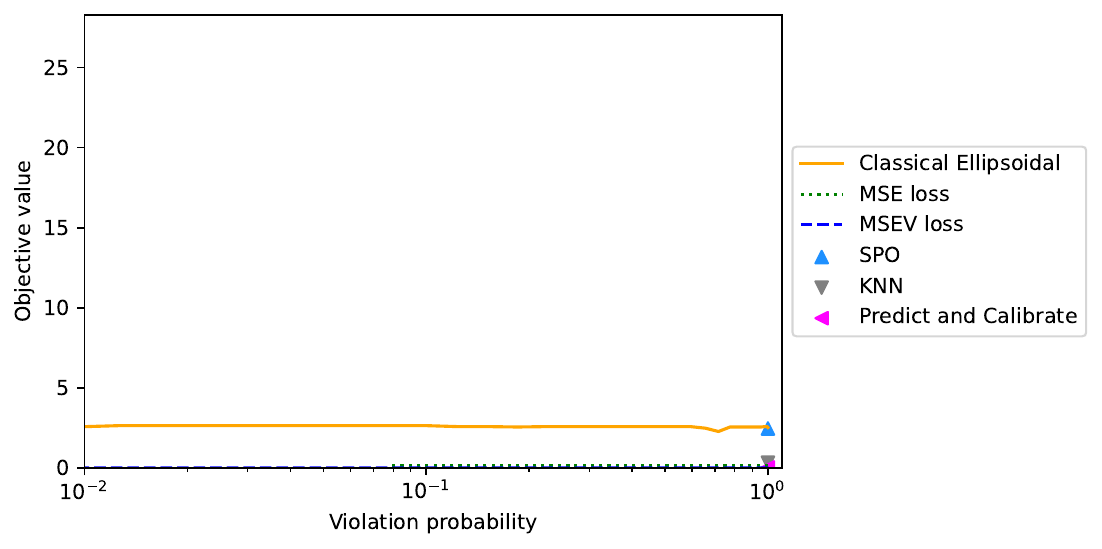} \\
 \multicolumn{3}{c}{$\text{deg} = 4$} \\[6pt]
$\dsp \Bar{\varepsilon} = 0.1$ & $\dsp \Bar{\varepsilon} = 0.3$ & $\dsp \Bar{\varepsilon} = 0.5$ \\[6pt]
\end{tabular}
\end{figure}

\end{document}